%% file: Main.tex
\documentclass[11pt]{article}
\pdfoutput=1

\usepackage{amsmath,amssymb}
\usepackage{bm}
\usepackage{natbib}
\usepackage[usenames]{color}
\usepackage{amsthm}

\usepackage{multirow} 
\usepackage{enumitem}

\usepackage[colorlinks,
linkcolor=red,
anchorcolor=blue,
citecolor=blue
]{hyperref}

\usepackage{setspace}
\usepackage[left=1in, right=1in, top=1in, bottom=1in]{geometry}

\usepackage{xcolor}

\usepackage{smile}

\def \CC

{\textcolor{red}}

\def \cite{\citet}
\newcommand{\la}{\langle}
\newcommand{\ra}{\rangle}
\title{\huge An Improved Analysis of Training Over-parameterized \\ Deep Neural Networks}
\author
{
	Difan Zou\thanks{Department of Computer Science, University of California, Los Angeles, CA 90095, USA; e-mail: {\tt knowzou@cs.ucla.edu}} 
	~~~and~~~
	Quanquan Gu\thanks{Department of Computer Science, University of California, Los Angeles, CA 90095, USA; e-mail: {\tt qgu@cs.ucla.edu}}
}
\date{}
\begin{document}

\maketitle

\begin{abstract}
A recent line of research has shown that gradient-based algorithms with random initialization can converge to the global minima of the training loss for over-parameterized (i.e., sufficiently wide) deep neural networks. However, the condition on the width of the neural network to ensure the global convergence is very stringent, which is often a high-degree polynomial in the training sample size $n$ (e.g., $O(n^{24})$). 
In this paper, we provide an improved analysis of the global convergence of (stochastic) gradient descent for training deep neural networks, which only requires a milder over-parameterization condition than previous work in terms of the training sample size and other problem-dependent parameters. The main technical contributions of our analysis include (a) a tighter gradient lower bound that leads to a faster convergence of the algorithm, and (b) a sharper characterization of the trajectory length of the algorithm. By specializing our result to two-layer (i.e., one-hidden-layer) neural networks, it also provides a milder over-parameterization condition than the best-known result in prior work.
\end{abstract}
\input{intro}

\input{problem}
\input{theory}
\input{conclusion.tex}

\appendix
\input{appendix.tex}

\bibliography{relu}
\bibliographystyle{ims}

\end{document}

%% file: intro.tex
\section{Introduction}
Recent study \citep{zhang2016understanding} has revealed that deep neural networks trained by gradient-based algorithms can fit training data with random labels and achieve zero training error. 
Since the loss landscape of training deep neural network is highly nonconvex or even nonsmooth, conventional optimization theory cannot explain why gradient descent (GD) and stochastic gradient descent (SGD) can find the global minimum of the loss function (i.e., achieving zero training error). To better understand the training of neural networks, there is a line of research \citep{tian2017analytical,brutzkus2017globally,du2017convolutional,li2017convergence,zhong2017recovery,du2018power,zhang2018learning} studying two-layer (i.e., one-hidden-layer) neural networks, where it assumes there exists a teacher network (i.e., an underlying ground-truth network) generating the output given the input, and casts neural network learning as weight matrix recovery for the teacher network. 
 However, these studies not only make strong assumptions on the training data, but also need special initialization methods that are very different from the commonly used initialization method \citep{he2015delving} in practice. \citet{li2018learning,du2018gradient} advanced this line of research by proving that under much milder assumptions on the training data, (stochastic) gradient descent can attain a global convergence for training over-parameterized (i.e.,sufficiently wide) two-layer ReLU network with widely used random initialization method \citep{he2015delving}. More recently, \citet{allen2018convergence,du2018gradientdeep,zou2018stochastic} generalized the global convergence results from two-layer networks to deep neural networks. However, there is a huge gap between the theory and practice since all these work \cite{li2018learning,du2018gradient,allen2018convergence,du2018gradientdeep,zou2018stochastic} require unrealistic over-parameterization conditions on the width of neural networks, especially for deep networks. In specific, in order to establish the global convergence for training two-layer ReLU networks, \citet{du2018gradient} requires the network width, i.e., number of hidden nodes, to be at least $\Omega(n^6/\lambda_0^4)$, where $n$ is the training sample size and $\lambda_0$ is the smallest eigenvalue of the so-called Gram matrix defined in \cite{du2018gradient}, which is essentially the neural tangent kernel \citep{jacot2018neural,chizat2018note} on the training data. 
 Under the same assumption on the training data,  \citet{wu2019global} improved the iteration complexity of GD in \citet{du2018gradient} from $O\big(n^2\log(1/\epsilon)/\lambda_0^2\big)$ to $O\big(n\log(1/\epsilon)/\lambda_0\big)$ and \citet{oymak2019towards} improved the over-parameterization condition to $\Omega(n\|\Xb\|_2^6/\lambda_0^4)$, where $\epsilon$ is the target error and $\Xb\in\RR^{n\times d}$ is the input data matrix. For deep ReLU networks, the best known result was established in \cite{allen2018convergence}, which requires the network width to be at least $\tilde\Omega(kn^{24}L^{12}\phi^{-8})$\footnote{Here $\tilde \Omega(\cdot)$ hides constants and the logarithmic dependencies on problem dependent parameters except $\epsilon$.} to ensure the global convergence of GD and SGD, where $L$ is  the number of hidden layers, $\phi$ is the minimum data separation distance and $k$ is the output dimension.

This paper continues the line of research, and improves the over-parameterization condition and the global convergence rate of (stochastic) gradient descent for training deep neural networks. 
In specific, under the same setting as in \cite{allen2018convergence}, we prove faster global convergence rates for both GD and SGD under a significantly milder condition on the neural network width. Furthermore, when specializing our result to two-layer ReLU networks, it also outperforms the best-known result proved in \cite{oymak2019towards}. The improvement in our result is due to the following two innovative proof techniques: (a) a tighter gradient lower bound, which leads to a faster rate of convergence for GD/SGD; and (b) a sharper characterization of the trajectory length for GD/SGD until convergence.

We highlight our main contributions as follows:
\begin{itemize}[leftmargin=*]
    \item We show that, with Gaussian random initialization \citep{he2015delving} on each layer, when the number of hidden nodes per layer is $\tilde \Omega\big(kn^8L^{12}\phi^{-4}\big)$, GD can achieve $\epsilon$ training loss within $\tilde O\big(n^2L^2\log(1/\epsilon)\phi^{-1}\big)$ iterations, where $L$ is the number of hidden layers, $\phi$ is the minimum data separation distance, $n$ is the number of training examples, and $k$ is the output dimension. Compared with the state-of-the-art result \citep{allen2018convergence}, our over-parameterization condition is milder by a factor of $\tilde \Omega(n^{16}\phi^{-4})$, and our iteration complexity is better by a factor of $\tilde O(n^4\phi^{-1})$.
    \item We also prove a similar convergence result for SGD. We show that with Gaussian random initialization \citep{he2015delving} on each layer, when the number of hidden nodes per layer is $\tilde \Omega\big(kn^{17}L^{12}B^{-4}\phi^{-8}\big)$, SGD can achieve $\epsilon$ expected training loss  within  $\tilde O\big(n^5\log(1/\epsilon)B^{-1}\phi^{-2}\big)$ iterations, where $B$ is the minibatch size of SGD. Compared with the corresponding results in \citet{allen2018convergence}, our results are strictly better by a factor of $\tilde \Omega(n^7B^5)$ and $\tilde O(n^2)$ respectively regarding over-parameterization condition and iteration complexity. 
    \item When specializing our results of training deep ReLU networks with GD to two-layer ReLU networks, it also outperforms the corresponding results  \citep{du2018gradient,wu2019global,oymak2019towards}. In addition, for training two-layer ReLU networks with SGD, we are able to show much better result than training deep ReLU networks with SGD.

\end{itemize}

For the ease of comparison, we summarize the best-known results \citep{du2018gradient,allen2018convergence,du2018gradientdeep,wu2019global,oymak2019towards} of training overparameterized neural networks with GD and compare with them in terms of over-parameterization condition and iteration complexity in Table \ref{table:comparison}. We will show in Section \ref{sec:main theory} that, under the assumption that all training data points have unit $\ell_2$ norm, which is the common assumption made in all these work \citep{du2018gradient,allen2018convergence,du2018gradientdeep,wu2019global,oymak2019towards},  $\lambda_0>0$ is equivalent to the fact that all training data are separated by some distance $\phi$, and we have $\lambda_0 = O(n^{-2}\phi)$ \citep{oymak2019towards}. Substituting $\lambda_0 = \Omega(n^{-2}\phi)$ into Table \ref{table:comparison}, it is evident that our result outperforms all the other results under the same assumptions.

\begin{table}[t]
	\centering
	\caption {Over-parameterization conditions and iteration complexities of GD for training overparamterized neural networks. 
	$\Kb^{(L)}$ is the Gram matrix for $L$-hidden-layer neural network \citep{du2018gradientdeep}. Note that the dimension of the output  is $k=1$ in \citet{du2018gradient,du2018gradientdeep,wu2019global,oymak2019towards}. 
	\label{table:comparison}} 

\begin{tabular}{cccccc}
    \toprule
    &\hspace{-5mm}Over-para. condition& \hspace{-3mm}Iteration complexity & \hspace{-2mm}Deep? &\hspace{-2mm}ReLU?
     \\
	\midrule
    \citet{du2018gradient}
	&\hspace{-5mm}$\Omega\Big(\frac{n^6}{\lambda_0^4}\Big)$&\hspace{-3mm}$O\Big(\frac{n^2\log(1/\epsilon)}{\lambda_0^2}\Big)$&\hspace{-2mm} no& \hspace{-2mm}yes\\
	\citet{wu2019global}
	&\hspace{-5mm}$\Omega\Big(\frac{n^6}{\lambda_0^4}\Big)$&\hspace{-3mm}$O\Big(\frac{n\log(1/\epsilon)}{\lambda_0^2}\Big)$&\hspace{-2mm} no& \hspace{-2mm}yes\\
     \citet{oymak2019towards}&\hspace{-5mm}$\Omega\Big(\frac{n\|\Xb\|_2^6}{\lambda_0^4}\Big)$&\hspace{-3mm} $O\Big(\frac{\|\Xb\|_2^2\log(1/\epsilon)}{\lambda_0}\Big)$&\hspace{-2mm}no&\hspace{-2mm} yes\\
    \citet{du2018gradientdeep} &\hspace{-5mm}$\Omega\Big(\frac{2^{O(L)}\cdot n^4}{\lambda_{\text{min}}^4(\Kb^{(L)})}\Big) $&\hspace{-3mm}$O\Big(\frac{2^{O(L)}\cdot n^2\log(1/\epsilon)}{\lambda_{\text{min}}^2(\Kb^{(L)})}\Big)$&\hspace{-2mm}yes& \hspace{-2mm}no\\
    \citet{allen2018convergence} &\hspace{-5mm}$\tilde\Omega\Big(\frac{kn^{24}L^{12}}{\phi^8}\Big)$&\hspace{-3mm}$O\Big(\frac{n^6L^2\log(1/\epsilon)}{\phi^2}\Big)$&\hspace{-2mm}yes& \hspace{-2mm}yes\\
    \textbf{This paper}& \hspace{-5mm}$\tilde\Omega\Big(\frac{kn^8L^{12}}{\phi^4}\Big)$&\hspace{-3mm}$O\Big(\frac{n^2L^2\log(1/\epsilon)}{\phi}\Big)$&\hspace{-2mm}yes& \hspace{-2mm}yes\\
	\bottomrule
\end{tabular}
\end{table}

\noindent\textbf{Notation}
For scalars, vectors and matrices, we use lower case, lower case bold face, and upper case bold face letters to denote them respectively. For a positive integer, we denote by $[k]$ the set $\{1,\dots,k\}$. For a vector $\xb = (x_1,\dots,x_d)^\top$ and a positive integer $p$, we denote by $\|\xb\|_p=\big(\sum_{i=1}^d |x_i|^p\big)^{1/p}$ the $\ell_p$ norm of $\xb$. In addition, we denote by $\|\xb\|_\infty = \max_{i=1,\dots,d} |x_i|$ the $\ell_\infty$ norm of $\xb$, and $\|\xb\|_0 = |\{x_i:x_i\neq 0,i=1,\dots,d\}|$ the $\ell_0$ norm of $\xb$. For a matrix $\Ab \in \RR^{m\times n}$, we denote by $\|\Ab\|_F$  the Frobenius norm of $\Ab$, $\|\Ab\|_2$ the spectral norm (maximum singular value), $\lambda_{\min}(\Ab)$ the smallest singular value, $\|\Ab\|_0$ the number of nonzero entries, and $\|\Ab\|_{2,\infty}$ the maximum $\ell_2$ norm over all row vectors, i.e., $\|\Ab\|_{2,\infty} = \max_{i=1,\dots,m}\|\Ab_{i*}\|_2$. For a collection of matrices $\Wb = \{\Wb_1,\dots,\Wb_L\}$, we denote $\|\Wb\|_F = \sqrt{\sum_{l=1}^L\|\Wb_l\|_F^2}$, $\|\Wb\|_2 = \max_{l\in[L]}\|\Wb_l\|_2$ and $\|\Wb\|_{2,\infty} = \max_{l\in[L]}\|\Wb_l\|_{2,\infty}$. Given two collections of matrices $\tilde\Wb = \{\tilde\Wb_1,\dots,\tilde\Wb_L\}$ and $\hat\Wb = \{\hat\Wb_1,\dots,\hat\Wb_L\}$, we define their inner product as $\la\tilde\Wb,\hat \Wb \ra = \sum_{l=1}^L\la\tilde\Wb_l,\hat \Wb_l\ra$. For two sequences $\{a_n\}$ and $\{b_n\}$, we use $a_n = O(b_n)$ to denote that $a_n\le C_1 b_n$ for some absolute constant $C_1> 0$, and use $a_n = \Omega (b_n)$ to denote that $a_n\ge C_2 b_n$ for some absolute constant $C_2>0$. In addition, we use $\tilde O(\cdot)$ and $\tilde \Omega(\cdot)$ to hide logarithmic factors. 

%% file: problem.tex
\section{Problem setup and algorithms}

In this section, we introduce the problem setup and the training algorithms.

Following \cite{allen2018convergence}, we consider the training of an $L$-hidden layer fully connected neural network, which takes $\xb\in\RR^{d}$ as input, and outputs $\yb \in \RR^k$. In specific, the neural network is a vector-valued function $\fb_{\Wb}: \RR^d \rightarrow \RR^m$, which is defined as
\begin{align*}
    \fb_{\Wb}(\xb) = \Vb \sigma( \Wb_L \sigma( \Wb_{L-1} \cdots \sigma( \Wb_1 \xb )\cdots )   ),
\end{align*}
where $\Wb_1\in\RR^{m\times d}$, $\Wb_2,\dots,\Wb_L\in\RR^{m\times m}$ denote the weight matrices for the hidden layers, and $\Vb\in\RR^{k\times m}$ denotes the weight matrix in the output layer, $\sigma(x) = \max\{0,x\}$ is the entry-wise ReLU activation function. In addition, we denote by $\sigma'(x) = \ind(x)$ the derivative of ReLU activation function and $\wb_{l,j}$ the weight vector of the $j$-th node in the $l$-th layer. 

Given a training set $\{(\xb_i,\yb_i)\}_{i=1,\dots,n}$ where $\xb_i \in \RR^d$ and $\yb_i\in \RR^k$, the empirical loss function for training the neural network is defined as
\begin{align}\label{eq:training loss}
L(\Wb): = \frac{1}{n}\sum_{i=1}^n\ell(\hat\yb_i,\yb_i), 
\end{align}
where $\ell(\cdot,\cdot)$ is the loss function, and $\hat\yb_i = \fb_{\Wb}(\xb_i)$. In this paper, for the ease of exposition, we follow \cite{allen2018convergence,du2018gradient,du2018gradientdeep,oymak2019towards} and consider square loss as follows
\begin{align*}
\ell(\hat \yb_i,\yb_i)  = \frac{1}{2}\|\yb_i - \hat{ \yb}_i\|_2^2,
\end{align*}
where $\hat \yb_i = \fb_{\Wb}(\xb_i)\in\RR^k$ denotes the output of  the neural network given input $\xb_i$. It is worth noting that our result can be easily extended to other loss functions such as cross entropy loss \citep{zou2018stochastic} as well. 


We will study both gradient descent and stochastic gradient descent as training algorithms, which are displayed in Algorithm \ref{alg:GD}. For gradient descent, we update the weight matrix $\Wb_l^{(t)}$ using full partial gradient $\nabla_{\Wb_l}L(\Wb^{(t)})$. For stochastic gradient descent, we update the weight matrix $\Wb_l^{(t)}$ using stochastic partial gradient
$ 1/B\sum_{s\in\cB^{(t)}}\nabla_{\Wb_l}\ell\big(\fb_{\Wb^{(t)}}(\xb_s),\yb_s\big)$,
where $\cB^{(t)}$ with $|\cB^{(t)}| = B$ denotes the minibatch of training examples at the $t$-th iteration. Both algorithms are initialized in the same way as \cite{allen2018convergence}, which is essentially the initialization method \citep{he2015delving} widely used in practice. In the remaining of this paper, we denote by
\begin{align*}
\nabla L(\Wb^{(t)}) = \{\nabla_{\Wb_l}L(\Wb^{(t)})\}_{l\in[L]}\quad \mbox{and}\quad \nabla \ell\big(\fb_{\Wb^{(t)}}(\xb_i),\yb_i\big) = \{\nabla_{\Wb_l} \ell\big(\fb_{\Wb^{(t)}}(\xb_i),\yb_i\big)\}_{l\in[L]}
\end{align*}
the collections of all partial gradients of $L(\Wb^{(t)})$ and $\ell\big(\fb_{\Wb^{(t)}}(\xb_i),\yb_i\big)$.

\begin{algorithm}[t]
	\caption{(Stochastic) Gradient descent  with Gaussian random initialization} \label{alg:GD}
	\begin{algorithmic}[1]
 		\STATE \textbf{input:} Training data $\{\xb_i,\yb_i\}_{i\in[n]}$, step size $\eta$, total number of iterations $T$, minibatch size $B$.
	    \STATE \textbf{initialization:} For all $l\in[L]$, each row of weight matrix $\Wb_l^{(0)}$ is independently generated from $\cN(0,2/m\Ib)$, each row of $\Vb$ is independently generated from $\cN(0,\Ib/d)$\\
	    
	    {\hrulefill \ \bf Gradient Descent}\ \hrulefill\\
	    
	    \FOR{$t=0,\dots,T$}
	    \STATE $\Wb_{l}^{(t+1)} = \Wb_l^{(t)} - \eta\nabla_{\Wb_l} L(\Wb^{(t)})$ for all $l\in[L]$
	    \ENDFOR
		\STATE \textbf{output:} 
		$\{\Wb^{(T)}_l\}_{l\in[L]}$
		
		 {\hrulefill\  \bf Stochastic Gradient Descent}\  \hrulefill\\
		
		 \FOR{$t=0,\dots,T$}
	    \STATE Uniformly sample a minibatch of training data $\cB^{(t)}\in [n]$
	    \STATE $\Wb_{l}^{(t+1)} = \Wb_l^{(t)} - \frac{\eta}{B}\sum_{s\in\cB^{(t)}}\nabla_{\Wb_l}\ell\big(\fb_{\Wb^{(t)}}(\xb_s),\yb_s\big)$ for all $l\in[L]$
	    \ENDFOR
		\STATE \textbf{output:} 
		$\{\Wb^{(T)}_l\}_{l\in[L]}$
		
	\end{algorithmic}

\end{algorithm}

%% file: theory.tex
\section{Main theory}\label{sec:main theory}

In this section, we present our main theoretical results. We make the following assumptions on the training data.

\begin{assumption}\label{assump:normalized_data}
For any $\xb_i$, it holds that $\|\xb_i\|_2 = 1$ and $(\xb_i)_{d} = \mu$, where $\mu$ is an positive constant.
\end{assumption}
The same assumption has been made in all previous work along this line \citep{du2018gradientdeep,allen2018convergence,zou2018stochastic,oymak2019towards}.  Note that requiring the norm of all training examples to be $1$ is not essential, and this assumption can be relaxed to be $\|\xb_i\|_2$ is lower and upper bounded by some constants.

\begin{assumption}\label{assump:seperate}
For any two different training data points $\xb_i$ and $\xb_j$, there exists a positive constant $\phi>0$ such that  $\|\xb_i - \xb_j\|_2\ge \phi$.
\end{assumption}
This assumption has also been made in \cite{allen2018rnn,allen2018convergence}, which is essential to guarantee zero training error for deep neural networks. It is a quite mild assumption for the regression problem as studied in this paper. Note that \citet{du2018gradientdeep} made a different assumption on training data, which requires the Gram matrix $\Kb^{(L)}$ (See their paper for details) defined on the $L$-hidden-layer networks is positive definite. However, their assumption is not easy to verify for neural networks with more than two layers. 

Based on Assumptions \ref{assump:normalized_data} and \ref{assump:seperate}, we are able to establish the global convergence rates of GD and SGD for training deep ReLU networks. We start with the result of GD for $L$-hidden-layer networks.

\subsection{Training $L$-hidden-layer ReLU networks with GD}

The global convergence of GD for training deep neural networks is stated in the following theorem.
\begin{theorem}\label{thm:gd}
Under Assumptions \ref{assump:normalized_data} and \ref{assump:seperate}, and suppose the number of hidden nodes per layer satisfies
\begin{align}\label{eq:condition_m_gd}
m = \Omega\big(kn^8L^{12}\log^3(m)/\phi^4\big).
\end{align}
Then if set the step size $\eta = O\big(k/(L^2m)\big)$, with probability at least $1-O(n^{-1})$, gradient descent is able to find a point that achieves $\epsilon$ training loss within 
\begin{align*}
T = O\big(n^2L^2\log(1/\epsilon)/\phi\big)
\end{align*}
iterations.
\end{theorem}
\begin{remark}
The state-of-the-art results for training deep ReLU network are provided by \citet{allen2018convergence}, where the authors showed that GD can achieve $\epsilon$-training loss within $O\big(n^{6}L^2\log(1/\epsilon)/\phi^2\big)$ iterations if the neural network width satisfies $m = \tilde \Omega\big(kn^{24}L^{12}/\phi^{8}\big)$. As a clear comparison, our result on the iteration complexity is better than theirs by a factor of $O(n^{4}/\phi)$, and our over-parameterization condition is milder than theirs by a factor of $\tilde \Omega(n^{16}/\phi^4)$. \citet{du2018gradientdeep} also proved the global convergence of GD for training deep neural network with smooth activation functions. As shown in Table \ref{table:comparison}, the over-parameterization condition and iteration complexity in \citet{du2018gradientdeep} have an exponential dependency on $L$, which is much worse than the polynomial dependency on $L$ as in \cite{allen2018convergence} and our result.
\end{remark}

We now specialize our results in Theorem \ref{thm:gd} to two-layer networks by removing the dependency on the number of hidden layers, i.e., $L$. We state this result in the following corollary.
\begin{corollary}\label{corollary:GDtwolayer}
Under the same assumptions made in Theorem \ref{thm:gd}. For training two-layer ReLU networks, if set the number of hidden nodes $m = \Omega\big(kn^8\log^3(m)/\phi^4\big)$ and step size $\eta = O(k/m)$, then with probability at least $1-O(n^{-1})$, GD is able to find a point that achieves $\epsilon$-training loss within $T = O\big(n^2\log(1/\epsilon)/\phi\big)$ iterations.    
\end{corollary}

For training two-layer ReLU networks, \citet{du2018gradient} made a different assumption on the training data to establish the global convergence of GD. Specifically, \citet{du2018gradient} defined a Gram matrix, which is also known as neural tangent kernel \citep{jacot2018neural},  based on the training data $\{\xb_i\}_{i=1,\dots,n}$ and assumed that the smallest eigenvalue of such Gram matrix is strictly positive. In fact, for two-layer neural networks, their assumption is equivalent to Assumption \ref{assump:seperate}, as shown in the following proposition.
\begin{proposition}\label{prop:assumption}
Under Assumption \ref{assump:normalized_data}, define the Gram matrix $\Hb\in\RR^{n\times n}$ as follows 
\begin{align*}
\Hb_{ij} = \EE_{\wb\sim\cN(0,\Ib)}[\xb_i^\top\xb_j\sigma'(\wb^\top\xb_i)\sigma'(\wb^\top\xb_j)],    
\end{align*}
then the assumption $\lambda_0 = \lambda_{\min}(\Hb)>0$ is equivalent to Assumption \ref{assump:seperate}. In addition, there exists a sufficiently small constant $C$ such that $\lambda_0\ge C\phi n^{-2}$.
\end{proposition}

\begin{remark}
According to Proposition \ref{prop:assumption}, we can make a direct comparison between our convergence results for two-layer ReLU networks in Corollary \ref{corollary:GDtwolayer} with those in \cite{du2018gradient,oymak2019towards}.
In specific,  as shown in Table \ref{table:comparison}, the iteration complexity and over-parameterization condition proved in \cite{du2018gradient} can be translated to $O(n^6\log(1/\epsilon)/\phi^2)$ and $\Omega(n^{14}/\phi^4)$ respectively under Assumption \ref{assump:seperate}. \citet{oymak2019towards} improved the result in \cite{du2018gradient} and the improved iteration complexity and over-parameterization condition  can be translated to $O\big(n^2\|\Xb\|_2^2\log(1/\epsilon)/\phi\big)$ \footnote{It is worth noting that $\|\Xb\|_2^2=O(1)$ if $d\lesssim n$, $\|\Xb\|_2^2 = O(n/d)$ if $\Xb$ is randomly generated, and $\|\Xb\|_2^2 = O(n)$ in the worst case.} and $\Omega\big(n^9\|\Xb\|_2^6/\phi^4\big)$ respectively, where $\Xb = [\xb_1,\ldots,\xb_n]^\top\in\RR^{d\times n}$ is the input data matrix. 
Our iteration complexity for two-layer ReLU networks is better than that in \cite{oymak2019towards} by a factor of $O(\|\Xb\|_2^2)$ \footnote{Here we set $k=1$ in order to match the problem setting in \cite{du2018gradient,oymak2019towards}.}, and the over-parameterization condition is also strictly milder than the that in \cite{oymak2019towards} by a factor of $O(n\|\Xb\|_2^6)$.


\end{remark}

\subsection{Extension to training $L$-hidden-layer ReLU networks with SGD}


Then we extend the convergence results of GD to SGD in the following theorem.
\begin{theorem}\label{thm:sgd}
Under Assumptions \ref{assump:normalized_data} and \ref{assump:seperate}, and suppose the number of hidden nodes per layer satisfies
\begin{align}\label{eq:condition_m_sgd}
m = \Omega\big(kn^{17}L^{12}\log^3(m)/(B^4\phi^8)\big).
\end{align}
Then if set the step size as $\eta = O\big( kB\phi/(n^3m\log(m))\big)$, with probability at least $1-O(n^{-1})$, SGD is able to achieve $\epsilon$ expected training loss within
\begin{align*}
T = O\big( n^5\log(m)\log^2{(1/\epsilon)}/(B\phi^2)\big)
\end{align*}
iterations.
\end{theorem}

\begin{remark}
We first compare our result with the state-of-the-art proved in \cite{allen2018convergence}, where they showed that SGD can converge to a point with $\epsilon$-training loss within $\tilde O\big(n^7\log(1/\epsilon)/(B\phi^2)\big)$ iterations if $m = \tilde\Omega\big(n^{24}L^{12}Bk/\phi^8\big)$. In stark contrast, our result on the over-parameterization condition is strictly better than it by a factor of $\tilde \Omega(n^7B^5)$, and our result on the iteration complexity is also faster by a factor of $O(n^2)$.
\end{remark}

Moreover, we also characterize the convergence rate and over-parameterization condition of SGD for training two-layer networks. Unlike the gradient descent, which has the same convergence rate and over-parameterization condition for training both deep and two-layer networks in terms of training data size $n$, we find that the over-parameterization condition of SGD can be further improved for training two-layer neural networks. We state this improved result in the following theorem.

\begin{theorem}\label{thm:sgd_single}
Under the same assumptions made in Theorem \ref{thm:sgd}. For two-layer ReLU networks, if set the number of hidden nodes and step size as
\begin{align*}
m = \Omega\big(k^{5/2}n^{11}\log^3(m)/(\phi^5B)\big), \quad \eta= O\big(kB\phi/(n^3m\log(m))\big),
\end{align*}
then with probability at least $1-O(n^{-1})$, stochastic gradient descent is able to achieve $\epsilon$ training loss within $T = O\big(n^5\log(m)\log(1/\epsilon)/(B\phi^2)\big)$ iterations.
\end{theorem}

\begin{remark}
From Theorem \ref{thm:sgd}, we can also obtain the convergence results of SGD for two-layer ReLU networks by choosing $L=1$. 
However, the resulting
over-parameterization condition is $m = \Omega\big(kn^{17}\log^3(m)B^{-4}\phi^{-8}\big)$, which is much worse than that in Theorem \ref{thm:sgd_single}. This is because for two-layer networks, the training loss enjoys nicer local properties around the initialization, which can be leveraged to improve the convergence of SGD. Due to space limit, we defer more details to Appendix \ref{appendix:proof_sgd_one}.

\end{remark}


\section{Proof sketch of the main theory}\label{sec:proof_main}
In this section, we provide the proof sketch for Theorems \ref{thm:gd}, and highlight our technical contributions and  innovative proof techniques. 

\subsection{Overview of the technical contributions}
The improvements in our result are mainly attributed to the following two aspects: (1) a tighter gradient lower bound leading to faster convergence; and (2) a sharper characterization of the trajectory length of the algorithm. 

We first define the following perturbation region based on the initialization,
\begin{align*}
\cB(\Wb^{(0)},\tau) = \{\Wb: \|\Wb_l - \Wb_l^{(0)}\|_2\le \tau \mbox{ for all } l\in[L]\},
\end{align*}
where $\tau>0$ is the preset perturbation radius for each weight matrix $\Wb_l$.

\textbf{Tighter gradient lower bound.}
By the definition of $\nabla L(\Wb)$, we have $\|\nabla L(\Wb)\|_F^2 = \sum_{l=1}^L\|\nabla_{\Wb_l}L(\Wb)\|_F^2\ge \|\nabla_{\Wb_L}L(\Wb)\|_F^2$. Therefore, we can focus on the partial gradient of $L(\Wb)$ with respect to the weight matrix at the last hidden layer. Note that we further have $\|\nabla_{\Wb_L}L(\Wb)\|_F^2= \sum_{j=1}^m \|\nabla_{\wb_{L,j}}L(\Wb)\|_2^2$, where
\begin{align*}
\nabla_{\wb_{L,j}}L(\Wb) = \frac{1}{n}\sum_{i=1}^n\la \fb_{\Wb}(\xb_i) - \yb_i, \vb_j\ra\sigma'\big(\la\wb_{L,j},\xb_{L-1,i}\ra\big)\xb_{L-1,i},
\end{align*}
and $\xb_{L-1,i}$ denotes the output of the $(L-1)$-th hidden layer with input $\xb_i$.
In order to prove the gradient lower bound, for each $\xb_{L-1,i}$, we introduce a region namely ``gradient region'', denoted by $\cW_j$, which is almost orthogonal to $\xb_{L-1,i}$. Then we prove two major properties of these $n$ regions $\{\cW_1,\dots,\cW_n\}$: (1) $\cW_i\cap\cW_j = \emptyset$ if $i\neq j$, and (2) if $\wb_{L,j}\in \cW_i$ for any $i$, with probability at least $1/2$, $\|\nabla_{\wb_{L,j}}L(\Wb)\|_2$ is sufficiently large. We visualize these ``gradient regions'' in Figure \ref{fig:grad1}. Since $\{\wb_{L,j}\}_{j\in[m]}$ are randomly generated at the initialization, in order to get a larger bound of $\|\nabla_{\Wb_L}L(\Wb)\|_F^2$, we hope the size of these ``gradient regions'' to be as large as possible. We take the union of the ``gradient regions'' for all training data, i.e., $\cup_{i=1}^n\cW_i$, which is shown in Figure \ref{fig:grad1}. As a comparison, \citet{allen2018convergence,zou2018stochastic} only leveraged  the ``gradient region'' for one training data point to establish the gradient lower bound, which is shown in Figure \ref{fig:grad0}. Roughly speaking, the size of ``gradient regions'' utilized in our proof is $n$ times larger than those used in \cite{allen2018convergence,zou2018stochastic}, which consequently leads to an $O(n)$ improvement on the gradient lower bound. The improved gradient lower bound is formally stated in the following lemma.

\begin{lemma}[Gradient lower bound]\label{lemma:grad_lower_bounds}
Let $\tau = O\big(\phi^{3/2}n^{-3}L^{-6}\log^{-3/2}(m)\big)$, then for all $\Wb\in\cB(\Wb^{(0)},\tau)$, with probability at least $1-\exp\big(O(m\phi/(dn)))$, it holds that
\begin{align*}
\|\nabla &L(\Wb)\|_F^2\ge O\big(m\phi L(\Wb)/(kn^2)\big).
\end{align*}
\end{lemma}

\begin{figure}[t]
    \centering
     \subfigure[``gradient region'' for $\{\xb_{L-1,i}\}_{i\in[n]}$]{\includegraphics[width=0.4\linewidth]{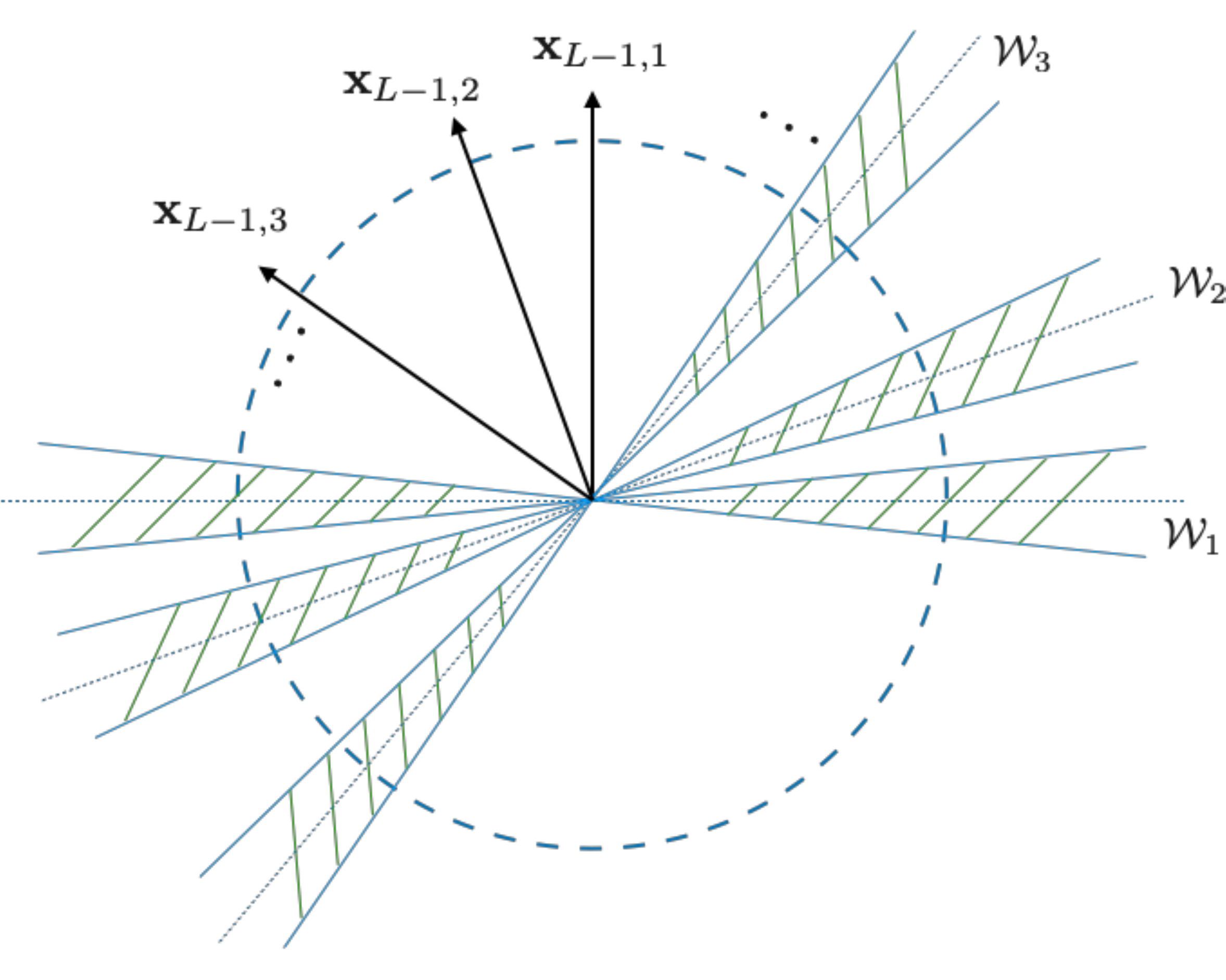}\label{fig:grad1}}
      \subfigure[``gradient region'' for $\xb_{L-1,1}$]{\includegraphics[width=0.4
    \linewidth]{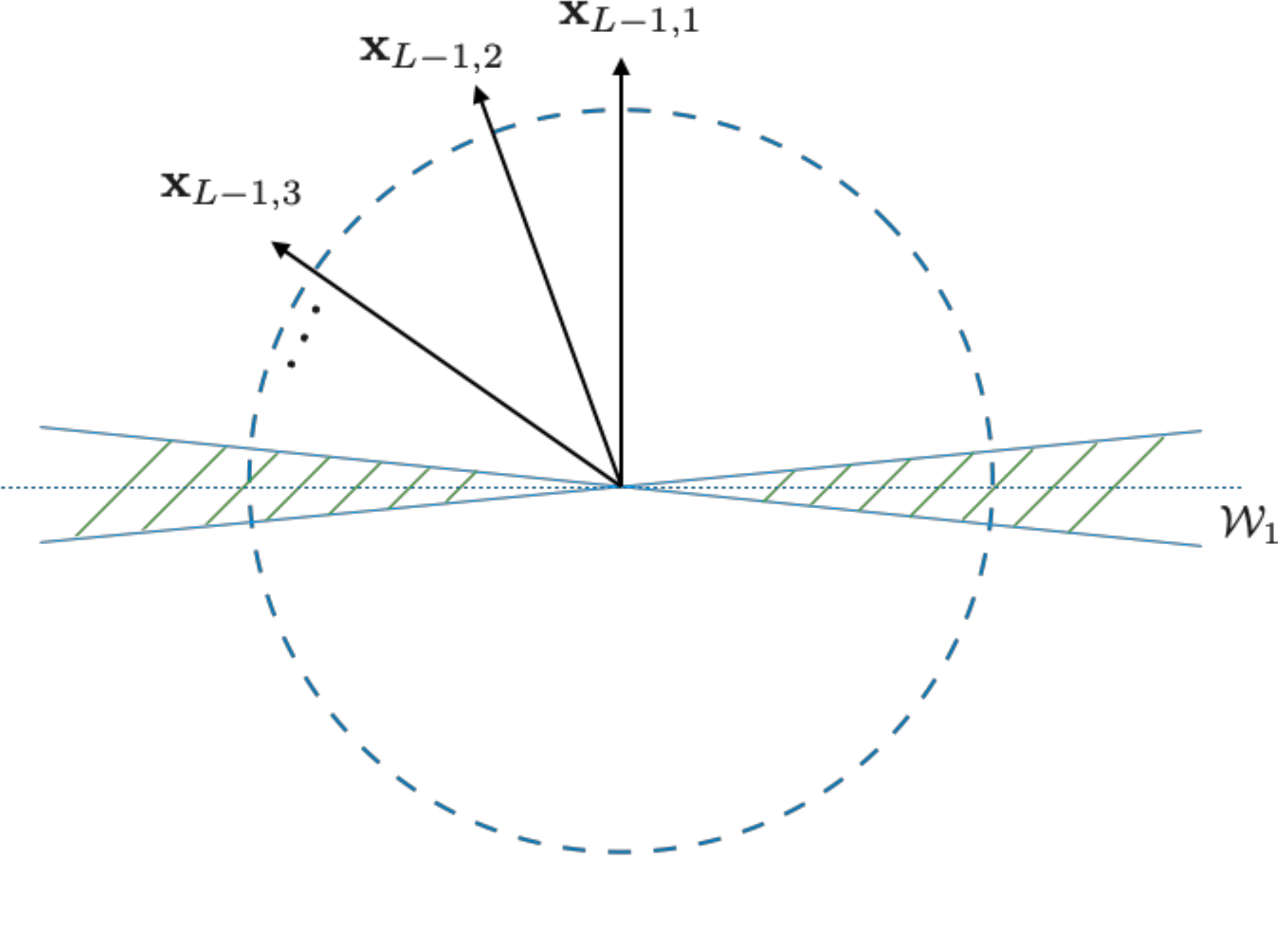}\label{fig:grad0}}
    \caption{(a): ``gradient region'' for all training data (b): ``gradient region'' for one training example.}
    \label{fig:my_label}
\end{figure}

\textbf{Sharper characterization of the trajectory length.}
The improved analysis of the trajectory length is motivated by the following observation: at the $t$-th iteration, the decrease of the training loss after one-step gradient descent is proportional to the gradient norm, i.e., $L(\Wb^{(t)}) - L(\Wb^{(t+1)})\propto \|\nabla L(\Wb^{(t)})\|_F^2$. In addition, the gradient norm $\|\nabla L(\Wb^{(t)})\|_F$  determines the trajectory length in the $t$-th iteration. Putting them together, we can obtain
\begin{align}\label{eq:Gu0001}
\|\Wb_l^{(t+1)} - \Wb_l^{(t)}\|_2= \eta\|\nabla_{\Wb_l}L(\Wb^{(t)})\|_2\le \sqrt{Ckn^2/(m\phi)}\cdot\Big(\sqrt{L(\Wb^{(t)}) }-\sqrt{L(\Wb^{(t+1)})}\Big),
\end{align}
where $C$ is an absolute constant. \eqref{eq:Gu0001} enables the use of telescope sum, which yields $\|\Wb_l^{(t)} - \Wb_l^{(0)}\|_2\le \sqrt{Ckn^2 L(\Wb^{(0)})/m\phi}$. In stark contrast, \citet{allen2018convergence} bounds the trajectory length as
\begin{align*}
\|\Wb_l^{(t+1)} - \Wb_l^{(t)}\|_2= \eta\|\nabla_{\Wb_l}L(\Wb^{(t)})\|_2\le \eta \sqrt{C'm L(\Wb^{(t)})/k},
\end{align*}
and further prove that  $\|\Wb_l^{(t)}-\Wb_l^{(0)}\|_2\le \sqrt{C'kn^6L^2(\Wb^{(0)})/(m\phi^2)}$ 
by taking summation over $t$, where $C'$ is an absolute constant.
Our sharp characterization of the trajectory length is formally summarized in the following lemma.

\begin{lemma}\label{lemma:dis_to_ini}
Assuming all iterates are staying inside the region $\cB(\Wb^{(0)},\tau)$ with $\tau = O\big(\phi^{3/2}n^{-3}L^{-6}\log^{-3/2}(m)\big)$, if set the step size $\eta = O\big(k/(L^2 m)\big)$, with probability least $1-O(n^{-1})$, the following holds for all $t\ge 0$ and $l\in[L]$,
\begin{align*}
\|\Wb_l^{(t)} - \Wb_l^{(0)}\|_2\le O\big( \sqrt{kn^2\log(n)/(m\phi)}\big).
\end{align*}
\end{lemma}

\subsection{Proof of Theorem \ref{thm:gd}}
Our proof road map can be organized in three steps: (i) prove that the training loss enjoys good curvature properties within the perturbation region $\cB(\Wb^{(0)},\tau)$; (ii) show that gradient descent is able to converge to global minima based on such good curvature properties; and (iii) ensure all iterates stay inside the perturbation region until convergence. 

\noindent\textbf{Step (i) Training loss properties.}
We first show some key properties of the training loss within $\cB(\Wb^{(0)},\tau)$, which are essential to establish the convergence guarantees of gradient descent.
\begin{lemma}\label{lemma:ini_value}
If $m\ge O(L\log(nL))$, with probability at least $1-O(n^{-1})$ it holds that $L(\Wb^{(0)})\le \tilde O(1)$.
\end{lemma}
Lemma \ref{lemma:ini_value} suggests that the training loss $L(\Wb)$ at the initial point does not depend on the number of hidden nodes per layer, i.e., $m$.



Moreover, the training loss $L(\Wb)$ is nonsmooth due to the non-differentiable ReLU activation function. Generally speaking, smoothness is essential to achieve linear rate of convergence for gradient-based algorithms. Fortunately, \citet{allen2018convergence} showed that the training loss satisfies locally semi-smoothness property,  which is summarized in the following lemma.
\begin{lemma}[Semi-smoothness \citep{allen2018convergence}] \label{lemma:semi_smooth}
Let 
\begin{align*}
\tau \in \big[\Omega\big(1/(k^{3/2}m^{3/2}L^{3/2}\log^{3/2}(m))\big),O\big(1/(L^{4.5}\log^{3/2}(m))\big)\big].
\end{align*}
Then for any two collections $\hat\Wb = \{\hat\Wb_l\}_{l\in[L]}$ and $\tilde \Wb = \{\tilde \Wb_l\}_{l\in[L]}$ satisfying $\hat\Wb,\tilde\Wb\in \cB(\Wb^{(0)},\tau)$, with probability at least $1-\exp(-\Omega(-m\tau^{3/2}L))$, there exist two constants $C'$ and $C''$ such that
\begin{align}\label{eq:semi_smooth}
L(\tilde \Wb)&\le L(\hat \Wb) + \la\nabla L(\hat \Wb),\tilde \Wb - \hat \Wb\ra \notag\\
&\quad+ C'\sqrt{L(\hat \Wb)}\cdot\frac{\tau^{1/3}L^2\sqrt{m\log(m)}}{\sqrt{k}}\cdot \|\tilde \Wb - \hat \Wb\|_2 + \frac{C''L^2m}{k}\|\tilde \Wb - \hat \Wb\|_2^2.
\end{align}

\end{lemma}
Lemma \ref{lemma:semi_smooth} is a rescaled version of Theorem 4 in \cite{allen2018convergence}, since the training loss $L(\Wb)$ in \eqref{eq:training loss} is divided by the training sample size $n$, as opposed to the training loss in \cite{allen2018convergence}. This lemma suggests that if the perturbation region is small, i.e., $\tau\ll 1$, the non-smooth term (third term on the R.H.S. of  \eqref{eq:semi_smooth}) is small and dominated by the gradient term (the second term on the the R.H.S. of  \eqref{eq:semi_smooth}). Therefore, the training loss behaves like a smooth function in the perturbation region and the linear rate of convergence can be proved. 

\noindent\textbf{Step (ii) Convergence rate of GD.}
Now we are going to establish the convergence rate for gradient descent under the assumption that all iterates stay inside the region $\cB(\Wb^{(0)},\tau)$, where $\tau$ will be specified later. 
\begin{lemma}\label{lemma:convergence_gd}
Assume all iterates stay inside the region $\cB(\Wb^{(0)},\tau)$, where $\tau = O\big(\phi^{3/2}n^{-3}L^{-6}\log^{-3/2}(m)\big)$. Then under Assumptions \ref{assump:normalized_data} and \ref{assump:seperate}, if set the step size $\eta = O\big(k/(L^2 m)\big)$, with probability least $1-\exp\big(-O(m\tau^{3/2}L)\big)$, it holds that
\begin{align*}
L(\Wb^{(t)})\le \bigg(1 - O\bigg(\frac{m\phi\eta}{kn^2}\bigg)\bigg)^t L(\Wb^{(0)}).
\end{align*}
\end{lemma}
Lemma \ref{lemma:convergence_gd} suggests that gradient descent is able to decrease the training loss to zero at a linear rate. 

\noindent\textbf{Step (iii) Verifying all iterates of GD stay inside the perturbation region.}
Then we are going to ensure that all iterates of GD are staying inside the required region $\cB(\Wb^{(0)},\tau)$. 
Note that we have proved the distance $\|\Wb_l^{(t)}-\Wb_l^{(0)}\|_2$ in Lemma \ref{lemma:dis_to_ini}. Therefore, it suffices to verify that such distance is smaller than the preset value $\tau$.
Thus, we can complete the proof of Theorem \ref{thm:gd} by verifying the conditions based on our choice of $m$. Note that we have set the required number of $m$ in \eqref{eq:condition_m_gd},
plugging \eqref{eq:condition_m_gd} into the result of Lemma \ref{lemma:dis_to_ini}, we have with probability at least $1-O(n^{-1})$, the following holds for all $t\le T$ and $l\in[L]$
\begin{align*}
\|\Wb_l^{(t)} - \Wb_l^{(0)}\|_2\le O\big(\phi^{3/2}{n^{-3}L^{-6}\log^{-3/2}(m)}\big),
\end{align*}
which is exactly in the same order of $\tau$ in Lemma \ref{lemma:convergence_gd}. Therefore, our choice of $m$ guarantees that all iterates are inside the required perturbation region. In addition, by Lemma \ref{lemma:convergence_gd}, in order to achieve $\epsilon$ accuracy, we require 
\begin{align}\label{eq:distance_convergence}
T\eta = O\big(kn^2\log\big(1/\epsilon\big)m^{-1}\phi^{-1}\big).
\end{align}
Then substituting our choice of step size $\eta = O\big(k/(L^2 m)\big)$ into \eqref{eq:distance_convergence} and applying Lemma \ref{lemma:ini_value}, we can get the desired result for $T$. 

%% file: conclusion.tex
\section{Conclusions and future work}
In this paper, we studied the global convergence of (stochastic) gradient descent for training over-parameterized ReLU networks, and improved the state-of-the-art results. 
Our proof technique can be also applied to prove similar results for other loss functions such as cross-entropy loss and other neural network architectures such as convolutional neural networks (CNN) \citep{allen2018convergence,du2018gradient} and ResNet \citep{allen2018convergence,du2018gradient,zhang2019training}. One important future work is to investigate whether the over-parameterization condition and the convergence rate can be further improved. Another interesting future direction is to explore the use of our proof technique to improve the generalization analysis of overparameterized neural networks trained by gradient-based algorithms \citep{allen2018generalization,cao2019generalization,arora2019fine}.

%% file: appendix.tex
\section{Proof of the Main Theory}
\label{appendix:proof_sgd_one}
\subsection{Proof of Proposition \ref{prop:assumption}}
We prove this proposition by two steps: (1) we prove that if there is no duplicate training data, it must hold that $\lambda_{\min}(\Hb)>0$; (2) we prove that if there exists at least one duplicate training data, we have $\lambda_{\min}(\Hb)=0$.

The first step can be done by applying Theorem 3 in \citet{du2018gradient}, where the author showed that if for any $i\neq j$, $\xb_i\nparallel \xb_j$, then it holds that $\lambda_{\min}(\Hb)>0$. Since under Assumption \ref{assump:normalized_data}, we have $\|\xb_i\|_2=\|\xb_j\|_2$. Then it can be shown that $\xb_i\neq\xb_j$  for all $i\neq j$ is an sufficient condition to $\lambda_{\min}(\Hb)$.

Then we conduct the second step. Clearly, if we have two training data with $\xb_i=\xb_j$, it can be shown that $\Hb_{ik} = \Hb_{jk}$ for all $k=1,\dots,n$. This immediately implies that there exist two identical rows in $\Hb$, which further suggests that $\lambda_{\min}(\Hb)=0$. 

The last argument can be directly proved by Lemma I.1 in \citet{oymak2019towards}, where the authors showed that $\lambda_0 = \lambda_{\min}(\Hb)\ge \phi/(100n^2)$.

By combining the above discussions, we are able to complete the proof. 

\subsection{Proof of Theorem \ref{thm:sgd}}
Now we sketch the proof of Theorem \ref{thm:sgd}. Following the same idea of proving Theorem \ref{thm:gd}, we split the whole proof into three steps.

\noindent\textbf{Step (i) Initialization and perturbation region characterization.}  
Unlike the proof for GD, in addition to the crucial gradient lower bound specified in Lemma \ref{lemma:grad_lower_bounds}, we also require the gradient upper bound, which is stated in the following lemma. 
\begin{lemma}[Gradient upper bounds \citep{allen2018convergence}]\label{lemma:grad_bounds}
Let $\tau = O\big(\phi^{3/2}n^{-3}L^{-6}\log^{-3/2}(m)\big)$, then for all $\Wb\in\cB(\Wb^{(0)},\tau)$, with probability at least $1-\exp\big(O(m\phi/(dn)))$, it holds that
\begin{align*}
\|\nabla L(\Wb)\|_F^2\le O\bigg(\frac{mL(\Wb)}{k}\bigg),\quad\|\nabla \ell(\fb_{\Wb}(\xb_i),\yb_i)\|_F^2\le O\bigg(\frac{m\ell(\fb_{\Wb}(\xb_i),\yb_i)}{k}\bigg).
\end{align*}
\end{lemma}
In later analysis, we show 
that the gradient upper bound will be exploited to bound the distance between iterates of SGD and its initialization. 
Besides, note that Lemmas \ref{lemma:ini_value} and \ref{lemma:semi_smooth} hold for both GD and SGD, we do not state them again in this part.

\noindent\textbf{Step (ii) Convergence rate of SGD.} 
Analogous to the proof for GD, the following lemma shows that SGD is able to converge to the global minima at a linear rate. 

\begin{lemma}\label{lemma:convergence_sgd}
Assume all iterates stay inside the region $\cB(\Wb^{(0)},\tau)$, where $\tau = O\big(\phi^{3}B^{3/2}n^{-6}L^{-6}\log^{-3/2}(m)\big)$. Then under Assumptions \ref{assump:normalized_data} and \ref{assump:seperate}, if set the step size $\eta = O\big(B\phi/(L^2 mn^2)\big)$, with probability least $1-\exp\big(-O(m\tau^{3/2}L)\big)$, it holds that
\begin{align*}
\EE[L(\Wb^{(t)})]\le \bigg(1 - O\bigg(\frac{m\phi\eta}{kn^2}\bigg)\bigg)^t L(\Wb^{(0)}).
\end{align*}
\end{lemma}

\noindent\textbf{Step (iii) Verifying all iterates of SGD stay inside the perturbation region.}
Similar to the proof for GD, the following lemma characterizes the distance from each iterate to the initial point for SGD.
\begin{lemma}\label{lemma:dis_to_ini_sgd}
Under the same assumptions made in Lemma \ref{lemma:convergence_sgd}, if set the step size $\eta= O\big(kB\phi/(n^3m\log(m))\big)$, suppose $m\ge O(T\cdot n)$, with probability at least $1-O(n^{-1})$, the following holds for all $t\le T$ and $l\in[L]$,
\begin{align*}
\|\Wb^{(t)}_l - \Wb_l^{(0)}\|_2\le O\big(k^{1/2}n^{5/2}B^{-1/2}m^{-1/2}\phi^{-1}\big).
\end{align*}
\end{lemma}

\begin{proof}[Proof of Theorem \ref{thm:sgd}]
Compared with Lemma \ref{lemma:dis_to_ini}, the trajectory length of SGD is much larger than that of GD. In addition, we require a much smaller step size to guarantee that the iterates do not go too far away from the initial point. This makes over-parameterization condition of SGD worse than that of GD. 

We complete the proof of Theorem \ref{thm:sgd} by verifying our choice of $m$ in \eqref{eq:condition_m_sgd}. By substituting \eqref{eq:condition_m_sgd} into Lemma \ref{lemma:dis_to_ini_sgd}, we have with probability at least $1-O(n^{-1})$, the following holds for all $t\le T$ and $l\in[L]$
\begin{align*}
\|\Wb_l^{(t)}-\Wb_l^{(0)}\|_2 =O\big(\phi^{3/2}B^{3/2}n^{-6}L^{-6}\log^{-3/2}(m)\big),
\end{align*}
which is exactly in the same order of $\tau$ in Lemma \ref{lemma:convergence_sgd}. Then by Lemma \ref{lemma:convergence_sgd}, we know that in order to achieve $\epsilon$ expected training loss, it suffices to set
\begin{align*}
T\eta  = O\big(kn^2m^{-1}\phi^{-1}\log(1/\epsilon)\big).
\end{align*} 
Then applying our choice of step size, i.e., $\eta= O\big(kB\phi/(n^3m\log(m))\big)$, we can get the desired result for $T$. This completes the proof.
\end{proof}

\subsection{Proof of Theorem \ref{thm:sgd_single}}
Before proving Theorem \ref{thm:sgd_single}, we first deliver the following two lemmas.
The first lemma states the upper bound of stochastic gradient in $\|\cdot\|_{2,\infty}$ norm.
\begin{lemma}\label{lemma:grad_bound_2infty}
With probability at least $1-O(m^{-1})$, it holds that \begin{align*}
\|\nabla \ell(\fb_{\Wb}(\xb_i),\yb_i)\|_{2,\infty}^2\le O\big(\ell(\fb_{\Wb}(\xb_i),\yb_i)\cdot \log(m)\big)
\end{align*}
for all $\Wb\in\RR^{m\times d}$ and $i\in[n]$.
\end{lemma}
The following lemma gives a different version of semi-smoothness for two-layer ReLU network.
\begin{lemma}[Semi-smoothness for two-layer ReLU network] \label{lemma:semi_smooth_single}
For any two collections $\hat \Wb = \{\hat\Wb_l\}_{l\in[L]}$ and $\tilde \Wb = \{\tilde \Wb_l\}_{l\in[L]}$ satisfying $\hat\Wb,\tilde \Wb\in \cB(\Wb^{(0)},\tau)$, with probability at least $1-\exp(-O(-m\tau^{2/3}))$, there exist two constants $C'$ and $C''$ such that
\begin{align*}
L(\tilde \Wb)&\le L(\hat \Wb) + \la\nabla L(\hat \Wb),\tilde \Wb - \hat \Wb\ra \notag\\
&\quad+ C'\sqrt{L(\hat \Wb)}\cdot\frac{\tau^{2/3}m\sqrt{\log(m)}}{\sqrt{k}}\cdot \|\tilde \Wb - \hat \Wb\|_{2,\infty} + \frac{C''m}{k}\|\tilde \Wb - \hat \Wb\|_2^2.
\end{align*}
\end{lemma}
It is worth noting that Lemma \ref{lemma:semi_smooth} can also imply a $\|\cdot\|_{2,\infty}$ norm based semi-smoothness result by applying the inequality $\|\tilde\Wb-\hat\Wb\|_2\le \|\tilde\Wb-\hat\Wb\|_F\le \sqrt{m}\|\tilde\Wb-\hat\Wb\|_{2,\infty}$. However, this operation will maintain the dependency on $\tau$, i.e., $\tau^{1/3}$, which is worse than that in Lemma \ref{lemma:semi_smooth_single} (e.g. $\tau^{2/3}$) since typically we have $\tau\ll1$.  Therefore, Lemma \ref{lemma:semi_smooth_single} is crucial to establish a better convergence guarantee for SGD in  training two-layer ReLU network.


\begin{proof}[Proof of Theorem \ref{thm:sgd_single}]
To simplify the proof, we use the following short-hand notation to define mini-batch stochastic gradient at the $t$-th iteration
\begin{align*}
\Gb^{(t)} = \frac{1}{|\cB^{(t)}|}\sum_{s\in\cB^{(t)}} \nabla \ell\big(\fb_{ \Wb^{(t)}}(\xb_s),\yb_s\big), 
\end{align*}
where $\cB^{(t)}$ is the minibatch of data indices with $|\cB^{(t)}| = B$. Then we bound its variance as follows,
\begin{align*}
\EE[\|\Gb^{(t)}-\nabla L(\Wb^{(t)})\|_F^2] &\le \frac{1}{B}\EE_s[\|\nabla \ell\big(\fb_{\Wb^{(t)}}(\xb_s),\yb_s\big)- \nabla L(\Wb^{(t)})\|_F^2]\notag\\
&\le \frac{2}{B}\big[\EE_s[\|\nabla \ell\big(\fb_{\Wb^{(t)}}(\xb_s),\yb_i\big)\|_F^2] + \|\nabla L(\Wb^{(t)})\|_F^2\big]\notag\\
&\le \frac{4L(\Wb^{(t)})}{Bk},
\end{align*}
where the expectation is taken over the random choice of training data and the second inequality follows from Young's inequality and the last inequality is by Lemma \ref{lemma:grad_bounds}. Moreover, we can further bound the expectation $\EE[\|\Gb^{(t)}\|_F^2]$ as follows,
\begin{align}\label{eq:bound_second_moment_g}
\EE[\|\Gb^{(t)}\|_F^2] \le 2\EE[\|\Gb^{(t)} - \nabla L(\Wb^{(t)})\|_F^2] + 2 \|\nabla L( \Wb^{(t)})\|_F^2\le \frac{8mL(\Wb^{(t)})}{Bk} + 2\|\nabla L(\Wb^{(t)})\|_F^2.   
\end{align}

By Lemma \ref{lemma:semi_smooth_single}, we have the following for one-step stochastic gradient descent
\begin{align*}
L(\Wb^{(t+1)})&\le L(\Wb^{(t)}) - \eta\la\nabla L(\Wb^{(t)}), \Gb^{(t)}\ra \notag\\
&\quad + C'\eta\sqrt{L\big(\Wb^{(t)}\big)}\cdot\frac{\tau^{2/3}m\sqrt{\log(m)}}{\sqrt{k}}\cdot \|\Gb^{(t)}\|_{2,\infty} + \frac{C''m\eta^2}{k}\cdot\|\Gb^{(t)}\|_{2}^2.
\end{align*}
Taking expectation conditioned on $\Wb^{(t)}$, we obtain
\begin{align}\label{eq:one_step_sgd_single}
\EE[L(\Wb^{(t+1)})|\Wb^{(t)}]&\le L(\Wb^{(t)}) - \eta\la\nabla L(\Wb^{(t)}), \Gb^{(t)}\ra \notag\\
&\quad + C'\eta\sqrt{L\big(\Wb^{(t)}\big)}\cdot\frac{\tau^{2/3}m\sqrt{\log(m)}}{\sqrt{k}}\cdot \EE[\|\Gb^{(t)}\|_{2,\infty}|\Wb^{(t)}]\notag\\
&\quad+ \frac{C''m\eta^2}{k}\cdot\EE[\|\Gb^{(t)}\|_{2}^2|\Wb^{(t)}].
\end{align}
By Lemma \ref{lemma:grad_bound_2infty}, with probability at least $1-O(m^{-1})$ we have the following upper bound on the quantity $\EE[\|\Gb^{(t)}\|_{2,\infty}|\Wb^{(t)}]$ for all $t=1,\dots,T$,
\begin{align*}
\EE[\|\Gb^{(t)}\|_{2,\infty}|\Wb^{(t)}]\le\EE[\|\nabla \ell(\fb_{\Wb^{(t)}}(\xb_i),\yb_i)\|_{2,\infty}|\Wb^{(t)}] \le O\big(\sqrt{L(\Wb^{(t)})\log(m)}\big). 
\end{align*}
Then based on Lemma \ref{eq:grad_lowerbound_gd}, plugging \eqref{eq:bound_second_moment_g} and the above inequality into \eqref{eq:one_step_sgd_single}, and set
\begin{align*}
\eta = O\bigg(\frac{k}{mn^2}\bigg)\quad \mbox{and}\quad
\tau = O\bigg(\frac{\phi^{3}}{n^3k^{3/4}\log^{3/2}(m)}\bigg).
\end{align*}
Then with proper adjustment of constants we can obtain
\begin{align*}
\EE[L(\Wb^{(t+1)})|\Wb^{(t)}]\le L(\Wb^{(t)}) - \frac{\eta}{2}\|\nabla L(\Wb^{(t)})\|_F^2\le \bigg(1-\frac{m\phi\eta}{2kn^2}\bigg)L(\Wb^{(t)}),
\end{align*}
where the last inequality follows from Lemma \ref{lemma:grad_lower_bounds}.
Then taking expectation on $\Wb^{(t)}$, we have with probability $1-O(m^{-1})$,
\begin{align}\label{eq:convergence_one_sgd}
\EE[L(\Wb^{(t+1)})]\le \bigg(1-\frac{m\phi\eta}{2kn^2}\bigg)\EE[L(\Wb^{(t)})]\le \bigg(1-\frac{m\phi\eta}{2kn^2}\bigg)^{t+1}\EE[L(\Wb^{(0)})]
\end{align}
holds for all $t>0$.
Then by Lemma \ref{lemma:dis_to_ini_sgd}, we know that if set $\eta= O\big(kB\phi/(n^3m\log(m))\big)$, with probability at least $1-O(n^{-1})$, it holds that
\begin{align*}
\|\Wb_l^{(t)} - \Wb_l^{(0)}\|_2\le O\bigg(\frac{k^{1/2}n^{5/2}}{B^{1/2}m^{1/2}\phi}\bigg),
\end{align*}
for all $t\le T$.
Then by our choice of $m$, it is easy to verify that with probability at least $1-O(n^{-1})-O(m^{-1}) = 1-O(n^{-1})$,
\begin{align*}
\|\Wb_l^{(t)} - \Wb_l^{(0)}\|_2\le  O\bigg(\frac{k^{1/2}n^{5/2}}{B^{1/2}\phi} \cdot \frac{\phi^4B^{1/2}}{k^{5/4}n^{11/2}\log^{3/2}(m)}\bigg)= \tau. 
\end{align*}
Moreover, note that in Lemma \ref{lemma:dis_to_ini_sgd} we set the step size as $\eta= O\big(kB\phi/(n^3m\log(m))\big)$ and \eqref{eq:convergence_one_sgd} suggests that we need
\begin{align*}
T\eta  = O\bigg(\frac{kn^2}{m\phi}\bigg)    
\end{align*}
to achieve $\epsilon$ expected training loss. Therefore we can derive the number of iteration as
\begin{align*}
T = O\bigg(\frac{n^5\log(m)\log(1/\epsilon)}{B\phi^2}\bigg).
\end{align*}
This completes the proof.
\end{proof}

\section{Proof of Lemmas in Section \ref{sec:proof_main} and Appendix \ref{appendix:proof_sgd_one}} \label{appendix:proof of technical lemmas}

\subsection{Proof of Lemma \ref{lemma:grad_lower_bounds}}

We first provide the following useful lemmas before starting the proof of Lemma \ref{lemma:grad_lower_bounds}.

The following lemma states that with high probability the norm of the output of each hidden layer is bounded by constants.
\begin{lemma}[\citep{zou2018stochastic}]\label{lemma:bound_output}
If $m\ge O(L\log(nL))$, with probability at least $1-\exp(-O(m/L))$, it holds that $1/2\le\|\xb_{l,i}\|_2\le2$ and $\big\|\xb_{l,i}/\|\xb_{l,i}\|_2 - \xb_{l,j}/\|\xb_{l,j}\|_2\big\|_2\ge \phi/2$ for all $i,j\in[n]$ and $l\in[L]$, where $\xb_{l,i}$ denotes the output of the $l$-th hidden layer given the input $\xb_i$.
\end{lemma}

\begin{lemma}\label{lemma:grad_lowerbound_initial}
Assume $m\ge \tilde O(n^{2}k^2\phi^{-1})$, then there exist an absolute constant $C > 0$ such that
with probability at least $1-\exp\big(-O(m\phi/(kn))\big)$, it holds that
\begin{align*}
\sum_{j=1}^m\bigg\|\frac{1}{n}\sum_{i=1}^n\la \ub_i, \vb_j\ra\sigma'\big(\la\wb_{L,j}^{(0)},\xb_{L-1,i}\ra\big)\xb_{L-1,i}\bigg\|_2^2\ge \frac{C \phi m\sum_{i=1}^n \|\ub_i\|_2^2}{kn^3}.
\end{align*}
\end{lemma}
If we set $\ub_i = \fb_{\Wb^{(0)}}(\xb_i) - \yb_i$, Lemma \ref{lemma:grad_lowerbound_initial} corresponds to the gradient lower bound at the initialization. Then the next step is to prove the bounds for all $\Wb$ in the required perturbation region. Before proceeding to our final proof, we present the following lemma that provides useful results regarding the neural network within the perturbation region.  
\begin{lemma}[\citep{allen2018convergence}]\label{lemma:perturbation}
Consider a collection of weight matrices $\tilde \Wb = \{\tilde\Wb_l\}_{l=1,\dots,L}$ such that $\tilde \Wb\in\cB(\Wb^{(0)},\tau)$, with probability at least $1-\exp(-O(m\tau^{2/3}L))$, there exists constants $C'$, $C''$ and $C'''$ such that
\begin{itemize}[leftmargin=*]
    \item $\|\tilde \bSigma_{L,i} - \bSigma_{L,i}\big\|_0\le C'\tau^{2/3}L$
    \item $\|\Vb(\tilde \bSigma_{L,i} - \bSigma_{L,i})\|_2\le C''\tau^{1/3}L^2\sqrt{m\log(m)}/\sqrt{k} $
    \item $\|\tilde\xb_{L-1,i}- \xb_{L-1,i}\|_2\le C'''\tau L^{5/2}\sqrt{\log(m)}$,
\end{itemize}
for all $i=1,\dots,n$, where $\xb_{L-1,i}$ and $\tilde \xb_{L-1,i}$ denote the outputs of the $L-1$-th layer of the neural network with weight matrices $\Wb^{(0)}$ and $\tilde\Wb$, and $ \bSigma_{L,i}$ and $\tilde\bSigma_{L,i}$ are diagonal matrices with $(\bSigma_{L,i})_{jj} = \sigma'(\la\wb_{L,j}^{(0)},\xb_{L-1}\ra)$ and $(\tilde\bSigma_{L,i})_{jj} = \sigma'(\la\tilde\wb_{L,j},\tilde\xb_{L-1}\ra)$ respectively.  
\end{lemma}
Now we are ready to prove the lower and upper bounds of the Frobenious norm of the gradient.

\begin{proof}[Proof of Lemma \ref{lemma:grad_lower_bounds}]
The upper bound of the gradient norm can be proved according to Theorem 3 in \cite{allen2018convergence}. We slightly modify their result since we consider average loss over all training examples while \citet{allen2018convergence} considers summation.

Then we focus on proving the lower bound.
Note that the gradient $\nabla_{\Wb_L} L(\tilde\Wb) $ takes form
\begin{align*}
\nabla_{\Wb_L} L(\tilde\Wb) = \frac{1}{n}\sum_{i=1}^n\bigg((\fb_{\tilde\Wb}(\xb_i) - \yb_i)^\top\Vb\tilde\bSigma_{L,i}\bigg)^\top\tilde\xb_{L-1,i}^\top,
\end{align*}
where $\tilde \bSigma_{L,i}$ is a diagonal matrix with $(\tilde \bSigma_{L,i})_{jj} = \sigma'(\tilde\wb_{L-1,j},\tilde \xb_{L-1,i})$ and $\tilde\xb_{l-1,i}$ denotes the output of the $l$-th hidden layer with input $\xb_i$ and model weight matrices $\tilde \Wb$.
Let $\vb_{j}^\top$ denote the $j$-th row of matrix $\Vb$, and define 
\begin{align*}
\tilde\Gb = \frac{1}{n}\sum_{i=1}^n\bigg((\fb_{\tilde\Wb}(\xb_i) - \yb_i)^\top\Vb\bSigma_{L,i}\bigg)^\top\xb_{L-1,i}^\top,
\end{align*}
where $ \bSigma_{L,i}$ is a diagonal matrix with $(\tilde \bSigma_{L,i})_{jj} = \sigma'(\wb^{(0)}_{L-1,j}, \xb_{L-1,i})$
Then by Lemma \ref{lemma:grad_lowerbound_initial}, we have with probability at least $1-\exp\big(-O(m\phi/(kn))\big)$, the following holds for any $\tilde \Wb$,
\begin{align*}
\|\tilde\Gb\|_F^2 &= \frac{1}{n^2}\sum_{j=1}^m\bigg\|\sum_{i=1}^n \la\fb_{\tilde \Wb}(\xb_i)-\yb_i,\vb_j\ra\sigma'(\la\wb_{L,j},\xb_{L-1,i}) \xb_{L-1,i}\bigg\|_2^2\notag\\
& \ge \frac{C_0\phi m \sum_{i=1}^n \|\fb_{\tilde \Wb}(\xb_i) - \yb_i\|_2^2}{kn^3},
\end{align*}
where $C_0$ is an absolute constant.
Then we have 
\begin{align*}
&\big\|\tilde\Gb - \nabla_{\Wb_L}L(\tilde \Wb)\big\|_F  \notag\\
&= \frac{1}{n^2}\bigg\|\sum_{i=1}^n\bigg((\fb_{\tilde\Wb}(\xb_i) - \yb_i)^\top\Vb\bSigma_{L,i}\bigg)^\top\xb_{L-1,i}^\top-\sum_{i=1}^n\bigg((\fb_{\tilde\Wb}(\xb_i) - \yb_i)^\top\Vb\tilde\bSigma_{L,i}\bigg)^\top\tilde\xb_{L-1,i}^\top\bigg\|_F\notag\\
& \le \frac{1}{n^2}\bigg[\bigg\|\sum_{i=1}^n\bigg((\fb_{\tilde\Wb}(\xb_i) - \yb_i)^\top\Vb(\bSigma_{L,i}-\tilde\bSigma_{L,i})\bigg)^\top\xb_{L-1,i}^\top\bigg\|_F \notag\\
&\quad+\bigg\|\sum_{i=1}^n\bigg((\fb_{\tilde\Wb}(\xb_i) - \yb_i)^\top\Vb\tilde\bSigma_{L,i}\bigg)^\top\big(\xb_{L-1,i}-\tilde \xb_{L-1,i}\big)^\top\bigg\|_F \bigg].
\end{align*}
By Lemmas \ref{lemma:bound_output} and \ref{lemma:perturbation}, we have
\begin{align*}
&\bigg\|\sum_{i=1}^n\bigg((\fb_{\tilde\Wb}(\xb_i) - \yb_i)^\top\Vb(\bSigma_{L,i}-\tilde\bSigma_{L,i})\bigg)^\top\xb_{L-1,i}^\top\bigg\|_F\notag\\
&\le \sum_{i=1}^n \big\|\fb_{\tilde\Wb}(\xb_i) - \yb_i\big\|_2\big\|\Vb(\bSigma_{L,i} - \tilde \bSigma_{L,i}) \big\|_2\|\xb_{L-1,i}\|_2\notag\\
&\le \frac{C_1\tau^{1/3}L^2\sqrt{m\log(m)}}{\sqrt{k}}\cdot \sum_{i=1}^n\big\|\fb_{\tilde \Wb(\xb_i)}-\yb_i\big\|_2,
\end{align*}
where the second inequality follows from Lemma \ref{lemma:perturbation}
and $C_1$ is an absolute constant. In addition, we also have
\begin{align*}
&\bigg\|\sum_{i=1}^n\bigg((\fb_{\tilde\Wb}(\xb_i) - \yb_i)^\top\Vb\tilde\bSigma_{L,i}\bigg)^\top\big(\xb_{L-1,i}-\tilde \xb_{L-1,i}\big)^\top\bigg\|_F\notag\\
&\le \sum_{i=1}^n \|\fb_{\tilde \Wb}(\xb_i) - \yb_i\|_2 \|\Vb\|_2\|\xb_{L-1,i} - \tilde \xb_{L-1,i}\|_2\notag\\
&\le \frac{C_2\tau L^{5/2}\sqrt{m\log(m)}}{\sqrt{k}}\cdot\sum_{i=1}^n \|\fb_{\tilde \Wb}(\xb_i) - \yb_i\|_2,
\end{align*}
where the second inequality follows from Lemma \ref{lemma:perturbation} and $C_2$ is an absolute constant. Combining the above bounds we have
\begin{align*}
\big\|\Gb - \nabla_{\Wb_L}L(\tilde \Wb)\big\|_F &\le \frac{\sum_{i=1}^n \|\fb_{\tilde \Wb}(\xb_i) - \yb_i\|_2}{n}\cdot\bigg(\frac{C_1\tau^{1/3}L^2\sqrt{m\log(m)}}{\sqrt{k}} + \frac{C_2\tau L^{5/2}\sqrt{m\log(m)}}{\sqrt{k}}\bigg)\notag\\
&\le \frac{\sum_{i=1}^n \|\fb_{\tilde \Wb}(\xb_i) - \yb_i\|_2}{n}\cdot \frac{C_3\tau^{1/3}L^2\sqrt{m\log(m)}}{\sqrt{k}},
\end{align*}
where the second inequality follows from the fact that $\tau\le O(L^{-4/3})$. Then by triangle inequality, we have the following lower bound of $\|\nabla_{\Wb_L}L(\tilde \Wb)\|_F$
\begin{align*}
\|\nabla_{\Wb_L}L(\tilde \Wb)\|_F&\ge \|\Gb\|_F - \|\Gb - \nabla_{\Wb_L}L(\tilde \Wb)\|_F\notag\\
&\ge \frac{C_0\phi^{1/2}m^{1/2}\sqrt{n\sum_{i=1}^n\|\fb_{\tilde \Wb}(\xb_i) - \yb_i\|_2^2}}{\sqrt{k}n^2} \notag\\
&\quad- \frac{\sum_{i=1}^n \|\fb_{\tilde \Wb}(\xb_i) - \yb_i\|_2}{n}\cdot \frac{C_3\tau^{1/3}L^2\sqrt{m\log(m)}}{\sqrt{k}}.
\end{align*}
By Jensen's inequality we know that $n\sum_{i=1}^n\|\fb_{\tilde \Wb}(\xb_i) - \yb_i\|_2^2\ge \big(\sum_{i=1}^n\|\fb_{\tilde \Wb}(\xb_i) - \yb_i\|_2\big)^2$. Then we set 
\begin{align*}
\tau = \frac{C_3\phi^{3/2}}{2C_0n^3L^6\log^{3/2}(m)} = O\bigg(\frac{\phi^{3/2}}{n^3L^6\log^{3/2}(m)}\bigg),
\end{align*}
and obtain
\begin{align*}
\|\nabla_{\Wb_L}L(\tilde \Wb)\|_F\ge \frac{C_0\phi^{1/2}m^{1/2}\sqrt{n\sum_{i=1}^n\|\fb_{\tilde \Wb}(\xb_i) - \yb_i\|_2^2}}{2\sqrt{k}n^2}.
\end{align*}
Then plugging the fact that $1/n\sum_{i=1}^n\|\fb_{\tilde \Wb}(\xb_i) - \yb_i\|_2^2 = L(\tilde \Wb)$, we are able to complete the proof.
\end{proof}

\subsection{Proof of Lemma \ref{lemma:dis_to_ini}}
\begin{proof}[Proof of Lemma \ref{lemma:dis_to_ini}]
Note that we assume that all iterate are staying inside the region $\cB\big(\Wb^{(0)},\tau\big)$, then by Lemma \ref{lemma:semi_smooth}, with probability at least $1- \exp(-O(m\tau^{2/3}L))$, we have the following after one-step gradient descent
\begin{align}\label{eq:one_step_descent_GD0}
L\big(\Wb^{(t+1)}\big)&\le L\big(\Wb^{(t)}\big) - \eta\|\nabla L\big(\Wb^{(t)}\big)\|_F^2 \notag\\
&\quad + C'\eta\sqrt{L\big(\Wb^{(t)}\big)}\cdot\frac{\tau^{1/3}L^2\sqrt{m\log(m)}}{\sqrt{k}}\cdot \|\nabla L\big(\Wb^{(t)}\big)\|_2 + \frac{C''L^2m\eta^2}{k}\|\nabla L\big(\Wb^{(t)}\big)\|_2^2.
\end{align}
We first choose the step size 
\begin{align*}
\eta = \frac{k}{4C''L^2m} = O\bigg(\frac{k}{L^2m}\bigg),
\end{align*}
then \eqref{eq:one_step_descent_GD0} yields
\begin{align*}
L\big(\Wb^{(t+1)}\big)&\le L\big(\Wb^{(t)}\big) - \frac{3\eta}{4}\|\nabla L\big(\Wb^{(t)}\big)\|_F^2  + C'\eta\sqrt{L\big(\Wb^{(t)}\big)}\cdot\frac{\tau^{1/3}L^2\sqrt{m\log(m)}}{\sqrt{k}}\cdot \|\nabla L\big(\Wb^{(t)}\big)\|_2\notag\\
&\le L\big(\Wb^{(t)}\big) - \eta\|\nabla L\big(\Wb^{(t)}\big)\|_F\bigg(\frac{\|\nabla L\big(\Wb^{(t)}\big)\|_F}{2} -C'\sqrt{L\big(\Wb^{(t)}\big)}\cdot\frac{\tau^{1/3}L^2\sqrt{m\log(m)}}{\sqrt{k}} \bigg),
\end{align*}
where we use the fact that $\|\nabla L\big(\Wb^{(t)}\big)\|_2\le \|\nabla L\big(\Wb^{(t)}\big)\|_F$. Then by Lemma \ref{lemma:grad_lower_bounds}, we know that with probability at least $1- \exp\big(-O(m\phi/(kn))\big)$
\begin{align}\label{eq:grad_lowerbound_gd}
\|\nabla_{\Wb} L(\Wb^{(t)})\|_F^2\ge\|\nabla_{\Wb_L} L(\Wb^{(t)})\|_F^2 \ge \frac{Cm \phi}{kn^2}L\big(\Wb^{(t)}\big),
\end{align}
where $C$ is an absolute constant.
Thus, we can choose the radius $\tau$ as
\begin{align}\label{eq:required_tau}
\tau = \frac{C^{3/2}\phi^{3/2}}{64n^{3}C'^{3}L^6\log^{3/2}(m)} = O\bigg(\frac{\phi^{3/2}}{n^3L^6\log^{3/2}(m)}\bigg),
\end{align}
and thus the following holds with probability at least $1-\exp(-O(m\tau^{2/3}L)) - \exp\big(-O(m\phi/(kn))\big) = 1-\exp(-O(m\tau^{2/3}L))$,
\begin{align}\label{eq:gd_contraction}
L\big(\Wb^{(t+1)}\big)&\le L\big(\Wb^{(t)}\big) - \frac{\eta}{2}\|\nabla L\big(\Wb^{(t)}\big)\|_F^2,
\end{align}
where the second inequality follows from \eqref{eq:grad_lowerbound_gd}. This completes the proof.
By triangle inequality, we have
\begin{align}\label{eq:distance_to_ini_GD}
\|\Wb^{(t)}_l-\Wb^{(0)}_l\|_2&\le \sum_{s=0}^{t-1}\eta\|\nabla_{\Wb_l} L\big(\Wb^{(s)}\big)\|_2\le \sum_{s=0}^{t-1}\eta\|\nabla L\big(\Wb^{(s)}\big)\|_F.
\end{align}
Moreover, we have
\begin{align*}
\sqrt{L(\Wb^{(s)})} - \sqrt{L(\Wb^{(s+1)})}& = \frac{L(\Wb^{(s)}) - L(\Wb^{(s+1)})}{\sqrt{L(\Wb^{(s)})}+\sqrt{L(\Wb^{(s+1)})}}\notag\\
&\ge \frac{\eta \|\nabla L(\Wb^{(s)})\|_F^2}{4\sqrt{L(\Wb^{(s)})}}\notag\\
&\ge \sqrt{\frac{Cm\phi}{kn^2}}\cdot \frac{\eta\|\nabla L(\Wb^{(s)})\|_F}{4},
\end{align*}
where the second inequality is by \eqref{eq:gd_contraction} and the fact that $L(\Wb^{(s+1)})\le L(\Wb^{(s)})$, and the last inequality follows from \eqref{eq:grad_lowerbound_gd}. Plugging the above result into \eqref{eq:distance_to_ini_GD}, we have with probability at least $1-\exp(-O(m\tau^{2/3}L))$,
\begin{align}\label{eq:dis_to_ini2}
\|\Wb_l^{(t)} - \Wb_l^{(0)}\|_2&\le \sum_{s=0}^{t-1}\eta\|\nabla L(\Wb^{(s)})\|_F\notag\\
&\le 4\sqrt{\frac{kn^2}{Cm\phi}}\sum_{s=0}^{t-1}\Big[\sqrt{L(\Wb^{(s)})} - \sqrt{L(\Wb^{(s+1)})}\Big]\notag\\
&\le 4\sqrt{\frac{kn^2}{Cm\phi}}\cdot \sqrt{L(\Wb^{(0)})}.
\end{align}

Note that \eqref{eq:dis_to_ini2} holds for all $l$ and $t$. Then apply Lemma \ref{lemma:ini_value}, we are able to complete the proof.

\end{proof}

\subsection{Proof of Lemma \ref{lemma:ini_value}}

\begin{proof}[Proof of Lemma \ref{lemma:ini_value}]
Note that the output of the neural network can be formulated as 
\begin{align*}
f_{\Wb^{(0)}}(\xb_i) = \Vb\xb_{L,i},
\end{align*}
where $\xb_{L,i}$ denotes the output of the last hidden layer with input $\xb_i$.
Note that each entry in $\Vb$ is i.i.d. generated from Gaussian distribution $\cN(0,1/k)$. Thus, we know that with probability at least $1-\delta$, it holds that $\|\Vb\xb_{L,i}\|_2\le \sqrt{\log(1/\delta)}\cdot\|\xb_{L,i}\|_2$. Then by Lemma \ref{lemma:bound_output} and union bound, we have $\|\Vb\xb_{L,i}\|_2\le 2\sqrt{\log(1/\delta)}$ for all $i\in[n]$ with probability at least $1-\exp(-O(m/L)) - n\delta$. Then we set $\delta = O(n^{-2})$ and use the fact that $m\ge O(L\log(nL))$, we have
\begin{align*}
f_{\Wb^{(0)}}(\xb_i) = \|\Vb\xb_{L,i}\|_2^2\le O(\log(n))
\end{align*}
for all $i\in[n]$ with probability at least $1-O(n^{-1})$. Then by our definition of training loss, it follows that
\begin{align*}
L(\Wb^{(0)}) &= \frac{1}{2n}\sum_{i=1}\|\fb_{\Wb^{(0)}}(\xb_i) - \yb_i\|_2^2   \notag\\
&\le \frac{1}{n}\sum_{i=1}\big[\|\fb_{\Wb^{(0)}}(\xb_i)\|_2^2 +\|\yb_i\|_2^2 \big] \notag\\
&\le O(\log(n))
\end{align*}
with probability at least $1-O(n^{-1})$, where the first inequality is by Young's inequality and we assume that $\|\yb_i\|_2=O(1)$ for all $i\in[n]$ in the second inequality. This completes the proof.
\end{proof}

\subsection{Proof of Lemma \ref{lemma:convergence_gd}}


\begin{proof}[Proof of Lemma \ref{lemma:convergence_gd}]
By \eqref{eq:gd_contraction}, we have
\begin{align}\label{eq:gd_converge}
L\big(\Wb^{(t+1)}\big)&\le L\big(\Wb^{(t)}\big) - \frac{\eta}{2}\|\nabla L\big(\Wb^{(t)}\big)\|_F^2\notag\\
&\le \bigg(1 - \frac{Cm\phi\eta}{2kn^2}\bigg)L\big(\Wb^{(t)}\big)\notag\\
&\le\bigg(1 - \frac{Cm\phi\eta}{2kn^2}\bigg)^{t+1} L\big(\Wb^{(0)}\big) ,
\end{align}
where the second inequality follows from \eqref{eq:grad_lowerbound_gd}. This completes the proof.
\end{proof}

\subsection{Proof of Lemma \ref{lemma:convergence_sgd}}


\begin{proof}[Proof of Lemma \ref{lemma:convergence_sgd}]
Let $\Gb^{(t)}$ denote the stochastic gradient leveraged in the $t$-th iteration, where the corresponding minibatch is defined as $\cB^{(t)}$. By Lemma \ref{lemma:semi_smooth}, we have the following inequality regarding one-step stochastic gradient descent
\begin{align*}
L(\Wb^{(t+1)})&\le L(\Wb^{(t)}) - \eta\la\nabla L(\Wb^{(t)}), \Gb^{(t)}\ra \notag\\
&\quad + C'\eta\sqrt{L\big(\Wb^{(t)}\big)}\cdot\frac{\tau^{1/3}L^2\sqrt{m\log(m)}}{\sqrt{k}}\cdot \|\Gb^{(t)}\|_2 + \frac{C''L^2m\eta^2}{k}\cdot\|\Gb^{(t)}\|_2^2.
\end{align*}
Then taking expectation on both sides conditioned on $\Wb^{(t)}$ gives
\begin{align}\label{eq:one_step_decrease_sgd1}
&\EE\big[L(\Wb^{(t+1)})\big|\Wb^{(t)}\big]\notag\\
&\le L(\Wb^{(t)}) - \eta\|\nabla L(\Wb^{(t)})\|_F^2  + C'\eta\sqrt{L\big(\Wb^{(t)}\big)}\cdot\frac{\tau^{1/3}L^2\sqrt{m\log(m)}}{\sqrt{k}}\cdot \EE\big[\|\Gb^{(t)}\|_2\big|\Wb^{(t)}\big] \notag\\
&\quad+ \frac{C''L^2m\eta^2}{k}\cdot\EE\big[\|\Gb^{(t)}\|_2^2\big|\Wb^{(t)}\big].
\end{align}
Note that given $\Wb^{(t)}$, the expectations on $\|\Gb^{(t)}\|_2$ and $\|\Gb^{(t)}\|_2^2$ are only taken over the random minibatch $\cB^{(t)}$. Then by \eqref{eq:bound_second_moment_g}, we have
\begin{align*}
\EE\big[\|\Gb^{(t)}\|_2\big|\Wb^{(t)}\big]^2\le \EE\big[\|\Gb^{(t)}\|_F^2\big|\Wb^{(t)}\big] \le \frac{8mL(\Wb^{(t)})}{Bk}+2\|\nabla L(\Wb^{(t)})\|_F^2.
\end{align*}
By \eqref{eq:grad_lowerbound_gd}, we know that there is a constant $C$ such that $\|\nabla L(\Wb^{(t)})\|_F^2\ge Cm\phi L(\Wb^{(t)})/(kn^2)$.
Then we set the step size $\eta$ and radius $\tau$ as follows
\begin{align*}
\eta &= \frac{Cd}{64C''L^2mn^2} = O\bigg(\frac{k}{L^2mn^2}\bigg)\notag\\
\tau &= \frac{C^{3}\phi^{3/2}B^{3}}{64^2n^{6}C'^{3}L^6\log^{3/2}(m)} = O\bigg(\frac{\phi^{3}B^{3/2}}{n^6L^6\log^{3/2}(m)}\bigg)
\end{align*}
Then \eqref{eq:one_step_decrease_sgd1} yields that
\begin{align}\label{eq:one_step_descrease_sgd2}
\EE\big[L(\Wb^{(t+1)})\big|\Wb^{(t)}\big]&\le L\big(\Wb^{(t)}\big) - \eta\|\nabla L\big(\Wb^{(t)}\big)\|_F^2 -   \frac{C''L^2m\eta^2}{k}\bigg(\frac{8n^2}{c\phi B}+2\bigg)\cdot \|\nabla L(\Wb^{(t)})\|_F^2\notag\\
&\quad -C'\eta\sqrt{L\big(\Wb^{(t)}\big)}\cdot\frac{\tau^{1/3}L^2\sqrt{m\log(m)}}{\sqrt{k}}\cdot\sqrt{\frac{8n^2}{c\phi B}+2}\cdot\|\nabla L(\Wb^{(t)})\|_F\notag\\
&\le L(\Wb^{(t)}) - \frac{\eta}{2}\|\nabla L(\Wb^{(t)})\|_F^2.
\end{align}
Then applying \eqref{eq:grad_lowerbound_gd} again and taking expectation over $\Wb^{(t)}$ on both sides of \eqref{eq:one_step_descrease_sgd2}, we obtain
\begin{align*}
\EE\big[L(\Wb^{(t+1)})\big]\le \bigg(1-\frac{Cm\phi\eta}{2kn^2}\bigg)\EE[L(\Wb^{(t)})]\le \bigg(1-\frac{Cm\phi\eta}{2kn^2}\bigg)^{t+1}L(\Wb^{(0)}).
\end{align*}
This completes the proof.

\end{proof}

\subsection{Proof of Lemma \ref{lemma:dis_to_ini_sgd}}
\begin{proof}[Proof of Lemma \ref{lemma:dis_to_ini_sgd}]
We prove this by standard martingale inequality. By Lemma \ref{lemma:semi_smooth} and our choice of $\eta$ and $\tau$, we have
\begin{align}\label{eq:ascent_one_step}
L(\Wb^{(t+1)}) \le L(\Wb^{(t)}) + 2\eta \|\nabla L(\Wb^{(t)})\|_F\cdot\|\Gb^{(t)}\|_2 +  \eta^2 \|\Gb^{(t)}\|_2^2.
\end{align}
By Lemma \ref{lemma:grad_bounds}, we know that there exists an absolute constant $C$ such that
\begin{align*}
\|\nabla L(\Wb^{(t)})\|_F^2 \le \frac{CmL(\Wb^{(t)})}{k} \mbox{ and } \|\Gb^{(t)}\|_F^2\le\frac{CmnL(\Wb^{(t)})}{Bk} ,
\end{align*}
where $B$ denotes the minibatch size. Then note that $\eta\le O(B/n)$, we have the following according to \eqref{eq:ascent_one_step}
\begin{align*}
L(\Wb^{(t+1)})\le\bigg(1+\frac{C'mn^{1/2}\eta}{B^{1/2}d}\bigg)L(\Wb^{(t)}) ,
\end{align*}
where $C'$ is an absolute constant. Taking logarithm on both sides further leads to
\begin{align*}
\log\big(L(\Wb^{(t+1)})\big)\le \log\big(L(\Wb^{(t)})\big) + \frac{C'mn^{1/2}\eta}{B^{1/2}d},    
\end{align*}
where we use the inequality $\log(1+x)\le x$. By  \eqref{eq:grad_lowerbound_gd} and \eqref{eq:one_step_descrease_sgd2}, we know that
\begin{align*}
\EE[L(\Wb^{(t+1)})|\Wb^{(t)}]\le L(\Wb^{(t)}) - \frac{\eta}{4}\|\nabla L(\Wb^{(t)})\|_F^2 \le \bigg(1 - \frac{C''m\phi\eta}{2kn^2}\bigg)L(\Wb^{(t)}).
\end{align*}
Then by Jensen's inequality and the inequality $\log(1+x)\le x$, we have
\begin{align*}
\EE\big[\log\big(L(\Wb^{(t+1)})\big)|\Wb^{(t)}\big]\le \log\big(\EE[L(\Wb^{(t+1)})|\Wb^{(t)}]\big)\le \log\big(L(\Wb^{(t)})\big) - \frac{C''m\phi\eta}{2kn^2},
\end{align*}
which further yields the following by taking expectation on $\Wb^{(t)}$ and telescope sum over $t$,
\begin{align}\label{eq:log_decrease}
\EE\big[\log\big(L(\Wb^{(t)})\big)\big]\le \log\big(L(\Wb^{(t)})\big) - \frac{C''m\phi\eta}{2kn^2}.
\end{align}
Therefore $\{L(\Wb^{(t)})\}_{t=0,1\dots,}$ is a supermartingale. By one-side Azuma's inequality, we know that with probability at least $1-\delta$, the following holds for any $t$ 
\begin{align}\label{eq:bound_path}
\log\big(L(\Wb^{(t)})\big)&\le \EE[\log\big(L(\Wb^{(t)})\big)] + \frac{C'mn^{1/2}\eta}{B^{1/2}d}\sqrt{2t\log(1/\delta)}\notag\\
&\le \log\big(L(\Wb^{(0)})\big) - \frac{tC''m\phi\eta}{2kn^2}+ \frac{C'mn^{1/2}\eta}{B^{1/2}d}\sqrt{2t\log(1/\delta)}\notag\\
&\le \log\big(L(\Wb^{(0)})\big) - \frac{tC''m\phi\eta}{4kn^2} + \frac{C'^2mn^3\log(1/\delta)\eta}{C''kB\phi},
\end{align}
where the second inequality is by \eqref{eq:log_decrease} and we use the fact that $-at + b\sqrt{t}\le b^2/(4a)$ in the last inequality.
Then we chose $\delta = O(m^{-1})$ and
\begin{align*}
\eta = \frac{\log(2)C''kB\phi}{C'^2mn^3\log(1/\delta)}= O\bigg(\frac{kB\phi}{n^3m\log(m)}\bigg).
\end{align*}
Plugging these into \eqref{eq:bound_path} gives 
\begin{align*}
\log\big(L(\Wb^{(t)})\big)\le \log\big(2L(\Wb^{(0)})\big)  -  \frac{tC''m\phi\eta}{4kn^2},
\end{align*}
which implies that
\begin{align}\label{eq:func_exp_decay}
L(\Wb^{(t)})\le 2L(\Wb^{(0)})\cdot \exp\bigg(-\frac{tC''m\phi\eta}{4kn^2}\bigg).
\end{align}
By Lemma \ref{lemma:grad_bounds}, we have
\begin{align}\label{eq:sto_grad_bound}
\|\Gb^{(t)}\|_2\le \|\Gb^{(t)}\|_F\le O\bigg(\frac{m^{1/2}n^{1/2}\sqrt{L(\Wb^{(t)})}}{B^{1/2}k^{1/2}}\bigg)
\end{align}
for all $t\le T$. Therefore, plugging \eqref{eq:sto_grad_bound} into \eqref{eq:func_exp_decay} and taking union bound over all $t\le T$, and apply the result in Lemma \ref{lemma:ini_value}, the following holds for all $t\le T$ with probability at least $1-O(T\cdot m^{-1})-O(n^{-1}) = 1-O(n^{-1})$,
\begin{align*}
\|\Wb^{(t)}_l - \Wb_l^{(0)}\|_2\le \sum_{s=0}^{t-1}\eta\|\Gb^{(t)}\|_2 \le  O\bigg(\frac{m^{1/2}n^{1/2}}{B^{1/2}k^{1/2}}\bigg)\cdot \sum_{s=0}^{t-1}\eta\sqrt{L(\Wb^{(s)})} \le \tilde O\bigg(\frac{k^{1/2}n^{5/2}}{B^{1/2}m^{1/2}\phi}\bigg),
\end{align*}
where the first inequality is by triangle inequality,  the second inequality follows from \eqref{eq:sto_grad_bound} and  the last inequality is by \eqref{eq:func_exp_decay} and Lemma \ref{lemma:ini_value}. This completes the proof.

\end{proof}

\subsection{Proof of Lemma \ref{lemma:grad_bound_2infty}}

\begin{proof}[Proof of Lemma \ref{lemma:grad_bound_2infty}]
We first write the formula of $\nabla \ell\big(f_\Wb(\xb_i), \yb_i\big)$ as follows
\begin{align*}
\nabla \ell\big(f_{\Wb}(\xb_i), \yb_i\big) = \big(f_{\Wb}(\xb_i) - y_i\big)^\top\Vb \bSigma_i\big)^\top\xb_{i}^\top.
\end{align*}
Since $ \bSigma_i$ is an diagonal matrix with $\big( \bSigma_i\big)_{jj}= \sigma'(\la \wb_j,\xb_i\ra)$. Therefore, it holds that
\begin{align}\label{eq:gradient_2inf_bound}
\|\nabla \ell(f_{\Wb}(\xb_i),\yb_i)\|_{2,\infty} & = \max_{j\in[m]} \la f_{\tilde\Wb}(\xb_i)-\yb_i, \vb_{j}\ra\cdot\|\xb_i\|_2\le \max_{j\in[m]} \| f_{\Wb}(\xb_i)-\yb_i\|_2 \|\vb_{j}\|_2,
\end{align}
where $\vb_j$ denotes the $j$-th column of $\Vb$ and we use the fact that $\|\xb_i\|_2 = 1$. Note that $\vb_j\sim \cN(0, \Ib/k)$, we have
\begin{align*}
\PP\big(\|\vb_j\|_2^2\ge O\big(\log(m)\big) \big) \le O(m^{-1}).
\end{align*}
Applying union bound over $\vb_1,\dots,\vb_m$, we have with probability at least $1-O(m^{-1})$,
\begin{align*}
\max_{j\in[m]}\|\vb_j\|_2\le O\big(\log^{1/2}(m)\big).
\end{align*}
Plugging this into \eqref{eq:gradient_2inf_bound} and applying the fact that $\|f_\Wb(\xb_i) - \yb_i\|_2 = \sqrt{\ell(\fb_\Wb(\xb_i),\yb_i)}$, we are able to complete the proof.
\end{proof}

\subsection{Proof of Lemma \ref{lemma:semi_smooth_single}}
Recall that the output of two-layer ReLU network can be formulated as
\begin{align*}
\fb_{\Wb}(\xb_i) = \Vb\bSigma_i\Wb\xb_i,   
\end{align*}
where $\bSigma_i$ is a diagonal matrix with only non-zero diagonal entry $(\bSigma_i)_{jj} = \sigma'(\wb_j^\top\xb_i)$. Then based on the definition of $L(\Wb)$, we have
\begin{align*}
&L(\tilde\Wb) - L(\hat \Wb)\notag\\ &=\frac{1}{2n}\sum_{i=1}^n\|\Vb\tilde \bSigma_i\tilde\Wb\xb_i - \yb_i\|_2^2 - \frac{1}{n}\sum_{i=1}^n\|\Vb\hat \bSigma_i\hat\Wb\xb_i - \yb_i\|_2^2\notag\\
& = \underbrace{\frac{1}{2n}\sum_{i=1}^n \big\la\Vb\hat\bSigma_i \hat \Wb\xb_i - \yb_i,\Vb\tilde\bSigma_i\tilde\Wb \xb_i - \Vb\hat\bSigma_i\hat\Wb \xb_i\big\ra}_{I_1} + \underbrace{\frac{1}{2n}\sum_{i=1}^n\big\|\Vb\tilde\bSigma_i\tilde \Wb\xb_i - \Vb\hat\bSigma_i\hat\Wb\xb_i\big\|_2^2}_{I_2}.
\end{align*}
Then we tackle the two terms on the R.H.S. of the above equation separately. Regarding the first term, i.e., $I_1$, we have
\begin{align*}
I_1 &= \frac{1}{2n}\sum_{i=1}^n \big\la\Vb\hat\bSigma_i \hat \Wb\xb_i - \yb_i,\Vb\hat\bSigma_i(\tilde\Wb  - \hat\Wb) \xb_i\big\ra \notag\\
&\quad+\frac{1}{2n}\sum_{i=1}^n \big\la\Vb\hat\bSigma_i \hat \Wb\xb_i - \yb_i,\Vb(\tilde \bSigma_i - \hat \bSigma_i)\tilde \Wb \xb_i\big\ra\notag\\
&\le \la\nabla L(\hat \Wb),\tilde \Wb - \hat \Wb\ra + \frac{1}{2n}\sum_{i=1}^n\sqrt{\ell(f_{\hat\Wb}(\xb_i),\yb_i)}\cdot \|\Vb(\tilde \bSigma_i - \hat \bSigma_i)\tilde \Wb\xb_i\|_2\notag\\
&\le \la\nabla L(\hat \Wb),\tilde \Wb - \hat \Wb\ra + \frac{\sqrt{L(\hat\Wb)}}{2}\cdot \|\Vb(\tilde \bSigma_i - \hat \bSigma_i)\tilde \Wb\xb_i\|_2,
\end{align*}
where the last inequality follows from Jensen's inequality.
Note that the non-zero entries in $\tilde \bSigma_i - \hat\bSigma_i$ represent the nodes, say $j$, satisfying $\text{sign}(\tilde \wb_j^\top\xb_i)\neq \text{sign}(\hat\wb_j^\top\xb_i)$, which implies $\big|\tilde\wb_j^\top\xb_i\big|\le \big|(\tilde\wb_j - \hat\wb_j)^\top\xb_i\big|$. Therefore, we have
\begin{align*}
\|\Vb(\tilde \bSigma_i - \hat \bSigma_i)\tilde \Wb\xb_i\|_2^2\le \|\Vb(\tilde \bSigma_i - \hat \bSigma_i)(\tilde \Wb-\hat\Wb)\xb_i\|_2^2.
\end{align*}
By Lemma \ref{lemma:perturbation}, we have $\|\tilde \bSigma_i - \hat \bSigma_i\|_0 \le\|\tilde \bSigma_i -  \bSigma_i\|_0 + \|\hat \bSigma_i -  \bSigma_i\|_0 = O(m\tau^{2/3})$. Then we define $\bar \bSigma_i$  as 
\begin{align*}
\big(\bar \bSigma_i\big)_{jk} = |\big(\tilde \bSigma_i-\hat \bSigma_i\big)_{jk}|\quad \mbox{ for all $j,k$}.
\end{align*}
Then we have
\begin{align*}
\|\Vb(\tilde \bSigma_i - \hat \bSigma_i)\tilde \Wb\xb_i\|_2 &\le \|\Vb(\tilde \bSigma_i - \hat \bSigma_i)\bar\bSigma_i(\tilde \Wb-\hat\Wb)\xb_i\|_2\notag\\
&\le \|\Vb(\tilde \bSigma_i - \hat\bSigma_i)\|_2\cdot \|\bar\bSigma_i(\tilde \Wb-\hat\Wb)\|_F\notag\\
&\le \|\Vb(\tilde \bSigma_i - \hat\bSigma_i)\|_2\cdot\|\bar\bSigma_i\|_0^{1/2}\cdot\|\tilde \Wb-\hat\Wb\|_{2,\infty}.
\end{align*}
By Lemma \ref{lemma:perturbation}, we have with probability $1-O(m\tau^{2/3})$
\begin{align*}
\|\Vb(\tilde \bSigma_i - \hat \bSigma_i)\tilde \Wb\xb_i\|_2\le O(m\sqrt{\log(m)}\tau^{2/3}k^{-1})\cdot \|\tilde\Wb - \hat \Wb\|_{2,\infty}.
\end{align*}
In what follows we are going to tackle the term $I_2$. Note that for each $i$, we have
\begin{align*}
\|\Vb\tilde \bSigma_i \tilde \Wb\xb_i - \Vb\hat \bSigma_i\hat\Wb\xb_i\|_2& = \|\Vb\hat\bSigma_i(\tilde \Wb - \hat\Wb)\xb_i\|_2 + \|\Vb(\tilde \bSigma_i - \hat \bSigma_i)\tilde \Wb\xb_i\|_2\notag\\
&\le \|\Vb\|_2\|\tilde \Wb - \hat\Wb\|_2+\|\Vb(\tilde \bSigma_i - \hat\bSigma_i)\|_2\cdot \|\tilde \Wb-\hat\Wb\|_2\notag\\
&= O(m^{1/2}/k^{1/2})\cdot \|\tilde \Wb - \hat\Wb\|_2,
\end{align*}
where the last inequality holds due to the fact that $\|\Vb\|_2 = O(m^{1/2}/k^{1/2})$ with probability at least $1-\exp(-O(m))$. This leads to $I_2 \le O(m/k)\cdot\|\tilde\Wb-\hat \Wb\|_2^2$.
Now we can put everything together, and obtain
\begin{align*}
L(\tilde \Wb) - L(\hat \Wb) &= I_1 + I_2\notag\\
&\le \la\nabla L(\hat \Wb),\tilde \Wb - \hat \Wb\ra + O(m\sqrt{\log(m)}\tau^{2/3}k^{-1/2})\cdot\sqrt{L(\hat\Wb)} \cdot\|\tilde\Wb - \hat \Wb\|_{2,\infty}\notag\\
&+O(m/k)\cdot \|\tilde \Wb - \hat \Wb\|_2^2.
\end{align*}
Then applying union bound on the inequality for $I_1$ and $I_2$, we are able to complete the proof.

\section{Proof of Technical Lemmas in Appendix \ref{appendix:proof of technical lemmas}}\label{appendix:proof of auxiliary lemmas}
\subsection{Proof of Lemma \ref{lemma:grad_lowerbound_initial}}
Let $\zb_1,\ldots,\zb_n\in\RR^d$ be $n$ vectors with $1/2\le \min_{i}\{\|\zb_i\|_2\}\le \max_{i}\{\|\zb_i\|_2\}\le 2$. Let $\bar \zb_i = \zb_i/\|\zb_i\|_2$
and assume $\min_{i,j}\|\bar\zb_i-\bar\zb_j\|_2\ge \tilde\phi$.
Then for each $\zb_i$, we construct an orthonormal matrix $\Qb_i = [\bar\zb_i, \Qb_i']\in\RR^{d\times d}$. Then consider a random vector $\wb\in\cN(0,\Ib)$,
it follows that $\ub_i: = \Qb^\top_i\wb\sim\cN(0,\Ib)$. Then we can  decompose $\wb$ as
\begin{align}\label{eq:decomposition_wb}
\wb = \Qb_i\ub_i = \ub_i^{(1)}\bar\zb_i + \Qb'_i\ub_i',
\end{align}
where $\ub_i^{(1)}$ denotes the first coordinate of $\ub_i$ and $\ub_i': = (\ub_i^{(2)},\dots,\ub_i^{(d)})$.
Then let $\gamma = \sqrt{\pi}\tilde\phi/(8n)$, we define the following set of $\wb$ based on $\zb_i$,
\begin{align*}
\cW_i = \big\{\wb: |\ub_i^{(1)}|\le \gamma, |\la\Qb_i'\ub_i',\bar\zb_i\ra|\ge 2\gamma \text{ for all } \bar\zb_j \mbox{ such that } j\neq i\big\}.
\end{align*}
Regarding the class of sets $\{\cW_1,\dots,\cW_n\}$, we have the following lemmas.
\begin{lemma}\label{lemma:set_wi}
For each $\cW_i$ and $\cW_j$, we have
\begin{align*}
\PP(\wb\in\cW_i)\ge \frac{\tilde \phi }{n\sqrt{128e}}  \quad \mbox{and}\quad  \cW_i\cap\cW_j = \emptyset.
\end{align*}  
\end{lemma}
Then we deliver the following two lemmas which are useful to establish the required lower bound.
\begin{lemma}\label{lemma:li2018}
For any $\ab = (a_1,\ldots,a_n)^\top \in \RR^{n}$, let
$\hb(\wb) = \sum_{i=1}^n a_i\sigma'(\la\wb,\zb_i\ra)\zb_i$ where $\wb\sim N(\mathbf{0},\Ib)$ is a Gaussian random vector. Then it holds that
\begin{align*}
\PP\big[\|\hb(\wb)\|_2\ge |a_i|/4\big|\wb\in\cW_i\big] \ge 1/2.
\end{align*}
\end{lemma}
Now we are able to prove Lemma \ref{lemma:grad_lowerbound_initial}.
\begin{proof}[Proof of Lemma \ref{lemma:grad_lowerbound_initial}]
We first prove the result for any fixed $\ub_1,\dots,\ub_n$.
Then we define $a_{i}(\vb_j) = \la\ub_i,\vb_j\ra$,  $\wb_{j} =\sqrt{m/2}\wb_{L,j}$ and 
\begin{align*}
\hb( \vb_j,\wb_j) = \sum_{i=1}^n a_{i}(\vb_j)\sigma'(\la\wb_j,\xb_{L-1,i}\ra)\xb_{L-1,i}.
\end{align*}
Then we define the event 
\begin{align*}
\cE_i = \big\{j\in [m]: \wb_j'\in\cW_i,\|\hb(\vb_j,\wb_j)\|_2\ge | a_{i}(\vb_j)|/4, |a_{i}(\vb_j)|\ge \|\ub_i\|_2/\sqrt{k}\big\}.
\end{align*}
By Lemma \ref{lemma:bound_output}, we know that with high probability $1/2\le\|\xb_{L-1,i}\|_2\le 2$ for all $i$ and $\big\|\xb_{L-1,i}/\|\xb_{L-1,i}\|_2 - \xb_{L-1,j}/\|\xb_{L-1,j}\|_2\big\|\ge \phi/2$ for all $i\neq j$. Then by Lemma \ref{lemma:set_wi} we know that $\cE_i\cap\cE_j = \emptyset$ if $i\neq j$ and \begin{align}\label{eq:prob_event}
\PP(j\in\cE_i) &= \PP\big[\|\hb(\vb_j,\wb_j)\|_2\ge | a_{i}(\vb_j)|/4|\wb_j'\in\cW_i\big]\cdot \PP\big[\wb_j'\in\cW_i\big]\cdot \PP\big[|a_{i}(\vb_j)|\ge \|\ub_i\|_2/\sqrt{k}\big]\notag\\
&\ge \frac{\phi}{64\sqrt{2}en},
\end{align}
where the first equality holds because $\wb_j$ and $\vb_j$ are independent, and the second inequality follows from Lemmas \ref{lemma:set_wi}, \ref{lemma:li2018} and the fact that $\PP(|a_{i}(\vb_j)|\ge \|\ub_i\|_2/\sqrt{k})\ge 1/2$. Then we have
\begin{align*}
\|\nabla_{\Wb_L}L(\Wb)\|_F^2 &=\frac{1}{n^2}\sum_{j=1}^m\|\hb(\vb_j,\wb_j)\|_2^2\notag\\
&\ge \frac{1}{n^2}\sum_{j=1}^m \|\hb(\vb_j,\wb_j)\|_2^2\sum_{s=1}^n\ind\big(j\in\cE_s\big)\notag\\
&\ge \frac{1}{n^2}\sum_{j=1}^m \sum_{s=1}^n\frac{ \|\ub_s\|_2^2}{16k}\ind\big(j\in\cE_s\big),
\end{align*}
where the second inequality holds due to the fact that
\begin{align*}
\|\hb(\vb_j,\wb_j)\|_2^2\ind\big(j\in\cE_s\big)&\ge \frac{a_s^2(\vb_j)}{16}\ind(|a_s(\vb_j)|\ge \|\ub_s\|_2/\sqrt{k})\cdot\ind(j\in\cE_s)\notag\\
&\ge \frac{\|\ub_s\|_2^2}{16k}\ind(j\in\cE_s),
\end{align*}
where the first inequality follows from the definition of $\cE_s$.
Then we further define
\begin{align*}
Z_j = \sum_{s=1}^n\frac{ \|\ub_s\|_2^2}{16k}\ind\big(j\in\cE_s\big),
\end{align*}
and provide the following results for $\EE[Z(\wb_j)]$ and $\text{var}[Z(\wb_j)]$
\begin{align*}
\EE[Z_j]  &= \sum_{s=1}^n\frac{\|\ub_s\|_2^2}{16k}\PP\big(j\in\cE_s\big), \qquad
\text{var}[Z(\wb)] = \sum_{s=1}^n\frac{\|\ub_s\|_2^4}{256k^2} \PP\big(j\in \cE_s\big)\big[1-\PP\big(j\in \cE_s\big)].
\end{align*}
Then by Bernstein inequality, with probability at least $1-\exp\big(-O\big(m\EE[Z(\wb)]/\max_{i\in[n]}\|\ub_i\|_2^2\big)\big)$, it holds that
\begin{align*}
\sum_{j=1}^m Z_j \ge \frac{m}{2}\EE[Z_j] \ge \sum_{i=1}^n\frac{ \|\ub_i\|_2^2}{32d}\cdot \frac{m\phi }{64\sqrt{2}en}=\frac{C \phi m\sum_{i=1}^n\|\ub_i\|_2^2}{kn},
\end{align*}
where the second inequality follows from \eqref{eq:prob_event} and $C = 1/(2096\sqrt{2}e)$ is an absolute constant. Therefore, with probability at least $1-\exp\big(-O(m\phi/(kn))\big) $ we have 
\begin{align*}
\sum_{j=1}^m\bigg\|\frac{1}{n}\sum_{i=1}^n\la \ub_i, \vb_j\ra\sigma'\big(\la\wb_{L,j},\xb_{L-1,i}\ra\big)\xb_{L-1,i}\bigg\|_2^2\ge \frac{1}{n^2}\sum_{j=1}^m Z(\wb_j)\ge \frac{C\phi m\sum_{i=1}^n\|\ub_i\|_2^2}{kn^3}.
\end{align*}
Till now we have completed the proof for one particular vector collection $\{\ub_i\}_{i=1,\dots,n}$. Then we are going to prove that the above inequality holds for arbitrary $\{\ub_i\}_{i=1,\dots,n}$ with high probability. Taking $\epsilon$-net over all possible vectors $\{\ub_1,\dots,\ub_n\}\in(\RR^d)^n$ and applying union bound, the above inequality holds with probability at least $1-\exp\big(-O(m\phi/(kn)) + nk\log(nk)\big)$. Since we have $m\ge \tilde O\big(\phi^{-1}n^2k^2\big)$, the desired result holds for all choices of $\{\ub_1,\dots,\ub_n\}$.

\end{proof}

\section{Proof of Auxiliary Lemmas in Appendix \ref{appendix:proof of auxiliary lemmas}}
\begin{proof}[Proof of Lemma \ref{lemma:set_wi}]
We first prove that any two sets $\cW_i$ and $\cW_j$ have not overlap region. Consider an vector $\wb\in\cW_i$ with the decomposition
\begin{align*}
\wb = \ub_i^{(1)}\bar\zb_i + \Qb_i'\ub_i'.
\end{align*}
Then based on the definition of $\cW_i$ we have,
\begin{align*}
\la\wb,\bar\zb_j\ra = \la\ub^{(1)}_i\bar\zb_i + \Qb_i'\ub_i', \bar\zb_j\ra = \ub_i^{(1)}\la\bar\zb_i,\bar\zb_j\ra + \la\Qb_i'\ub_i',\bar\zb_j\ra.
\end{align*}
Since $\wb\in \cW_i$, we have $|\ub_i^{(1)}|\le \gamma$ and $|\la\Qb'\ub_i',\bar\zb_j\ra|\ge 2\gamma$. Therefore, note that $|\la\bar\zb_i,\bar\zb_j\ra|\le 1$, it holds that
\begin{align}\label{eq:property_wi}
|\la\wb,\bar\zb_j\ra|\ge  \big||\la\Qb_i'\ub_i',\bar\zb_j\ra| - |\ub_i^{(1)}|\big|> \gamma.
\end{align}
Note that set $\cW_j$ requires $|\ub_j^{(1)}| = \la\wb,\bar\zb_j\ra\le \gamma$, which conflicts with \eqref{eq:property_wi}. This immediately implies that $\cW_i\cap\cW_j = \emptyset$. 

Then we are going to compute the probability $\PP(\wb\in\cW_i)$. Based on the parameter $\gamma$, we define the following two events
\begin{align*}
    \cE_1(\gamma) = \big\{|\ub_i^{(1)}|\le \gamma \big\},~
    \cE_2(\gamma) = \big\{ |\la\Qb_i'\ub_i',\bar\zb_j\ra|\ge 2\gamma \text{ for all } \bar\zb_j, j\neq i \big\}.
\end{align*}
Evidently, we have $\PP(\wb\in\cW_i) = \PP(\cE_1) \PP(\cE_2)$. Since $\ub_i^{(1)}$ is a standard Gaussian random variable, we have
\begin{align*}
    \PP(\cE_1) = \frac{1}{\sqrt{2\pi }}\int_{-\gamma}^{\gamma} \exp\bigg( -\frac{1}{2}x^2 \bigg) \mathrm{d} x \geq \sqrt{\frac{2}{\pi e}} \gamma.
\end{align*}
Moreover, by definition, for any $j=1,\ldots, n$ we have
\begin{align*}
    \la \Qb_i'\ub_i' , \bar\zb_j \ra \sim N\big[ 0, 1 - (\bar\zb_i^\top \bar\zb_j)^2 \big].
\end{align*}
Note that for any $j\neq i$ we have $\|\bar\zb_i - \bar\zb_j\|_2\ge \tilde \phi$, then it follows that
$$
|\la \zb_i , \zb_j \ra| \leq 1 - \tilde\phi^2/2,
$$
and if $\tilde\phi^2 \leq 2$, then
\begin{align*}
    1 - (\bar\zb_i^\top \bar\zb_j)^2 \geq \tilde\phi^2 - \tilde\phi^4/4 \geq \tilde\phi^2 / 2.
\end{align*}
Therefore for any $j\neq i$, 
\begin{align*}
    \PP[ | \la \Qb_i'\ub_i' , \bar\zb_j \ra | < 2\gamma ] 
    = \frac{1}{\sqrt{2\pi }}\int_{-2[1 - (\bar\zb_i^\top \bar\zb_j)^2]^{-1/2}\gamma}^{2[1 - (\bar\zb_i^\top \bar\zb_j)^2]^{-1/2} \gamma} \exp\bigg( -\frac{1}{2}x^2 \bigg) \mathrm{d} x 
    \leq  \sqrt{\frac{8}{\pi}} \frac{\gamma}{[1 - (\bar\zb_i^\top \bar\zb_j)^2]^{1/2}}\leq \frac{4}{\sqrt{\pi}}\gamma \tilde\phi^{-1}
    .
\end{align*}
By union bound over $[n]$, we have
\begin{align*}
    \PP(\cE_2) = \PP[ | \la \Qb_i'\ub_i' , \bar\zb_j \ra | \geq 2\gamma, j\in \cI ] \geq 1 -  \frac{4}{\sqrt{\pi}} n \gamma \tilde\phi^{-1}.
\end{align*}
Therefore we have
\begin{align*}
    \PP(\wb\in\cW_i) \geq \sqrt{\frac{2}{\pi e}} \gamma \cdot \bigg( 1 - \frac{4}{\sqrt{\pi}} n \gamma \tilde\phi^{-1} \bigg).
\end{align*}
Plugging $\gamma = \sqrt{\pi} \tilde\phi / (8n)$, it holds that $\PP(\cE) \geq \tilde\phi / ( \sqrt{128e} n)$.  This completes the proof.
\end{proof}

\begin{proof}[Proof of Lemma \ref{lemma:li2018}]
Recall the decomposition of $\wb$ in \eqref{eq:decomposition_wb},
\begin{align*}
\wb = \ub_i^{(1)}\bar\zb_i + \Qb_i'\ub_i'.
\end{align*}
Define the event $\cE_i:=\{\wb\in\cW_i\}$. Then conditioning on $\cE_i$, we have
\begin{align}\label{eq:grad_lower_formula}
\hb(\wb) &= \sum_{i=1}^n a_i\sigma'(\la\wb,\zb_i\ra)\zb_r\notag\\
&= a_i\sigma'(\ub_i^{(1)})\zb_i + \sum_{j\neq i}a_j\sigma'\big(\ub^{(1)}_{i}\la\bar\zb_i,\zb_j\ra+\la\Qb_i'\ub_i',\zb_j\ra\big)\zb_j\notag\\
& = a_i\sigma'(\ub_i^{(1)})\zb_i + \sum_{j\neq i}a_j\sigma'\big(\la\Qb_i'\ub_i',\zb_j\ra\big)\zb_j
\end{align}
where the last equality follows from the fact that conditioning on event $\cE_i$, for all $j\neq i$, it holds that $|\la\Qb_i'\ub_i',\zb_j\ra|\ge 2\gamma\|\zb_j\|_2 \ge |\ub_i^{(1)}|\|\zb_j\|_2 \ge |\ub_i^{(1)} \la\zb_i,\zb_j\ra|$. We then consider two cases: $\ub_i^{(1)} > 0$ and $\ub_i^{(1)} < 0$, which occur equally likely conditioning on the event $\cE_i$. Let $u_1>0$ and $u_2<0$ denote $\ub_i^{(1)}$ in these two cases,
we have
\begin{align*}
\PP\bigg[\|\hb(\wb)\|_2\ge \inf_{{u_1>0,  u_2<0}}\max\big\{\big\|\hb(u_1\zb_i + \Qb_i'\ub_i' )\big\|_2,\big\|\hb(u_2\zb_i + \Qb_i'\ub_i')\big\|_2\big\} \bigg| \cE_i\bigg]\ge 1/2 .
\end{align*}
By the inequality  $\max\{\|\ab\|_2,\|\bbb\|_2\}\ge \|\ab-\bbb\|_2/2$, we have
\begin{align}\label{eq:prob_gradient_lower}
\PP\bigg[\|\hb(\wb)\|_2\ge \inf_{{u_1>0,  u_1<0}} \big\|\hb(u_1\zb_i + \Qb_i'\ub_i' ) - \hb(u_2\zb_i + \Qb_i'\ub_i') \big\|_2 / 2 \bigg| \cE_i\bigg]\ge 1/2 .
\end{align}
For any $u_1> 0$ and $u_2 < 0$, denote $\wb_1 = u_1\zb_i + \Qb_i'\ub_i'$, $\wb_2 = u_2\zb_i + \Qb_i'\ub_i'$. We now proceed to give lower bound for $\|\hb(\wb_1) - \hb(\wb_2)\|_2$. 
By \eqref{eq:grad_lower_formula}, we have
\begin{align}\label{eq:two_grad_difference}
\|\hb(\wb_1) - \hb(\wb_2)\|_2 = \|a_i\zb_i\|_2 \ge a_i/2,
\end{align}
where we use the fact that $\|\zb_i\|_2\ge1/2$.
Plugging this back into \eqref{eq:prob_gradient_lower}, we have 
\begin{align*}
\PP\big[\|\hb(\wb)\|_2\ge |a_i|/4\big|\cE_i\big]\ge 1/2.
\end{align*}
This completes the proof.
\end{proof}